\documentclass{article}

\PassOptionsToPackage{sort}{natbib}
\usepackage[preprint]{neurips_2026}

\usepackage[utf8]{inputenc} 
\usepackage[T1]{fontenc}  
\usepackage{hyperref} 
  \hypersetup{
  	colorlinks = true,
  	linkcolor = blue,
  	anchorcolor = blue,
  	citecolor = blue,
  	filecolor = blue,
  	urlcolor = blue
  }
\usepackage{url}      
\usepackage{booktabs}    
\usepackage{amsfonts}    
\usepackage{nicefrac}    
\usepackage{microtype}   
\usepackage{xcolor}     

\usepackage{marginnote}
\usepackage{xcolor}
\usepackage{float}
\usepackage[subrefformat=parens]{subcaption}
\usepackage{amssymb}
\usepackage{amstext}
\usepackage{amsmath}
\usepackage{amscd}
\usepackage{amsthm}
\usepackage{amsfonts}
\usepackage{enumerate}
\usepackage{graphicx}
\usepackage{latexsym}
\usepackage{mathrsfs}
\usepackage{mathtools}
\usepackage{cases}
\usepackage{verbatim}
\usepackage{bm}
\usepackage{tikz}
\usepackage{caption}
\usepackage{cleveref}
\usepackage{tikz}
\usetikzlibrary{matrix}
\usepackage{comment}
\usepackage{tikz-cd}
\usepackage{pb-diagram}
\usepackage[colorinlistoftodos]{todonotes}
\usepackage{url}
\crefname{example}{Example}{Examples}

\newcommand{\norm}[1]{\|#1\|}
\newcommand{\abs}[1]{|#1|}
\newcommand{\grad}{\nabla}

\renewcommand{\div}{\operatorname{div}}
\newcommand{\st}{\,:\,}
\newcommand{\R}{\mathbb{R}}
\newcommand{\bv}{\mathrm{BV}}
\newcommand{\RI}{\mathbb R \cup \{+\infty\}}
\DeclareMathOperator{\sign}{sign}

\newtheorem{lemma}{Lemma}
\newtheorem{theorem}{Theorem}
\newtheorem{corollary}{Corollary}
\newtheorem{proposition}{Proposition}
\newtheorem{definition}{Definition}
\newtheorem{example}{Example}

\newtheorem{remark}{Remark}

\title{Approximation of Maximally Monotone Operators :
\\
A Graph Convergence Perspective}

\author{%
 Takashi Furuya \\
 Doshisha University, RIKEN AIP \\
 \texttt{tfuruya@mail.doshisha.ac.jp} \\
 \AND
 Yury Korolev \\
 Department of Mathematical Sciences \\
 University of Bath,  UK \\
 \texttt{ymk30@bath.ac.uk}
 \And
 Takaharu Yaguchi \\
 Kyushu University, RIKEN AIP \\
 \texttt{yaguchi@imi.kyushu-u.ac.jp} \\
}

\begin{document}

\maketitle

\begin{abstract} 

Operator learning has been highly successful for continuous mappings between infinite-dimensional spaces, such as PDE solution operators.
However, many operators of interest—including differential operators—are discontinuous or set-valued, and lie outside classical approximation frameworks.
We propose a paradigm shift by formulating approximation via graph convergence (Painlevé–Kuratowski convergence), which is well-suited for closed operators.
We show that uniform and $L^p$ approximation are fundamentally inadequate in this setting. 
Focusing on maximally monotone operators, we prove that any such operator can be approximated in the sense of local graph convergence by continuous encoder–decoder architectures, and further construct structure-preserving approximations that retain maximal monotonicity via resolvent-based parameterizations.


\end{abstract}

\section{Introduction}
\label{sec:intro}

Operator learning has emerged as a powerful paradigm for approximating mappings between infinite-dimensional spaces, with typical applications in partial differential equations (PDEs) or inverse problems \cite{boulle2024mathematical, kovachki2024operator, lu2019deeponet, nelsen2025operator}. 
Most existing approaches focus on approximating solution operators of PDEs, mapping e.g. boundary conditions or coefficients to corresponding solutions. 
In contrast, a central perspective of this work is to learn the underlying differential operator itself, which defines the governing equation. Such methods constitute another important class of applications of machine learning to science \cite{Both2021-fg, Chen2021-rw, Long2018-rm, Unknown2018-sz, Hou2025-oy}.
This viewpoint is fundamentally different, as differential operators are typically discontinuous (unbounded)
and may even be set-valued \cite{Acary2008-gp, Wang1997-ff}, and thus fall outside classical approximation frameworks in operator learning. 
However, differential operators are naturally characterized as \emph{closed operators}, i.e. those with a closed graph.

In this work we therefore focus on closed operators as targets. 
This class includes not only differential operators, but also many important operators arising in applications.
For instance, inverses of compact operators and subdifferential operators of convex functionals fall into this class. 
Subdifferentials play a central role in variational problems \cite{bauschke2011convex,rockafellar1998variational}, 
while inverse operators are the subject of rich inverse problems and regularization theory \cite{engl1996regularization, scherzer2009variational, benning2018modern}. 

A key difficulty is that classical approximation frameworks in operator learning rely on continuity or $L^p$-type structures, which fail to capture the behavior of closed operators. 
As a result, such operators cannot, in general, be approximated under local uniform or $L^p$ convergence. 
This raises a fundamental question:
\[
\textbf{How should one formulate approximation when the target operator is closed?}
\]

To address the lack of continuity, we advocate changing the topology of approximation. 
A natural candidate is \emph{graph convergence} (also known as Painlev\'e--Kuratowski convergence \cite[Sec. 2.2]{adly2022preservation}), 
which is well-established in functional and variational analysis \cite{rockafellar1998variational}.

The key advantage of graph convergence is that it ensures stability of the learned input--output relation under limiting operations, in a manner consistent with the notion of closed operators. 
In particular, any accumulation point of approximate input--output pairs remains valid for the limiting operator. 
This property is essential in applications where solutions are defined through limits, such as PDEs.
Moreover, graph convergence naturally accommodates set-valued operators and situations where the domains of the approximating operators and the target operator differ. 
In contrast, classical notions such as uniform or $L^p$ convergence are not well-defined in these settings. 
These features are particularly important in ill-posed inverse problems.

However, the class of closed operators is too broad to admit a meaningful universal approximation theory under graph convergence 
and we provide simple examples of closed operators that cannot be approximated by continuous operator models, not only in the uniform or $L^2$ sense, but also in the sense of graph convergence (\cref{ex:uniform-conv,ex:Lp-conv,ex:graph-conv}). 
This reveals a fundamental structural barrier: without additional structure, even graph convergence fails for closed operators. 

In this work, we therefore consider maximally monotone operators~\cite{brezis1973ope}, which provide a natural and practically relevant subclass of closed operators where meaningful approximation becomes possible.
The most prominent example is the subdifferential of a convex functional, which includes many practically relevant differential operators such as the $p$-Laplace operator (which corresponds to the $p$-Dirichlet energy $J(u) = \frac1p\int \norm{\grad u(x)}^p \, dx$). More examples are given in \cref{ex-max-mono}, see also~\cite{bungert2020asymptotic, benedetti2024differential}. 

A key advantage of this class is that it admits well-behaved resolvent mappings, which are Lipschitz continuous and defined on the entire space. This enables stable approximation via Yosida regularization \cite{bauschke2011convex, rockafellar1998variational} and provides a natural bridge to continuous operator models. Moreover, maximal monotonicity is preserved under graph convergence \cite{adly2022preservation}, making it fully consistent with our approximation framework. 
In contrast, as shown in Example~\ref{ex:graph-conv}, there exist closed and monotone but not maximally monotone operators that cannot be approximated by any sequence of continuous operators even in graph convergence. 
This suggests that maximal monotonicity is a natural and
possibly near-minimal structural condition for approximability,
at least in the sense of graph convergence.
Finally, even within the class of maximally monotone operators, approximation in uniform or $L^p$ convergence is generally impossible (see Example~\ref{ex:Lp-conv}). 
This further justifies graph convergence as the appropriate notion of approximation in our setting.

\subsection{Related works}

A large body of recent work establishes approximation guarantees under the assumption that the target operator is continuous or belongs to $L^p$, 
typically proving convergence in the local uniform topology \cite{kovachki2021universal, kovachki2023neural, lanthaler2025nonlocality, godeke2025new} 
or in $L^p$-type norms \cite{bhattacharya2021model, lanthaler2022error, korolev2022two}. 
In contrast, our setting focuses on maximally monotone operators, which may be discontinuous, set-valued, and not belonging to a Bochner space. 
As a result, classical universality results in operator learning are not directly applicable in this setting.

Learning monotone or firmly non-expansive structures has also been considered in the context of inverse problems. 
\cite{pesquet2021learning} learns neural networks representing resolvents of maximally monotone operators, while \cite{bredies2024learning} proposes a data-driven approach for constructing firmly non-expansive operators. 
\cite{belkouchi2025learning} learns monotone neural networks via a penalization-based loss. 
However, these works do not address approximation of maximally monotone operators from the viewpoint of graph convergence.

Graph convergence of maximally monotone operators has been extensively studied in variational analysis and nonlinear PDEs; see, e.g., \cite{adly2022preservation,chiado1990g}. 
In particular, \cite{adly2022preservation} studies the preservation of maximal monotonicity under graph convergence. 
In contrast, our goal is not to deduce maximal monotonicity of approximating operators from graph convergence, but rather to construct approximations of a target maximally monotone operator using continuous encoder--decoder models, which are canonical architectures in operator learning \cite{kovachki2024operator, godeke2025new}. 
Importantly, our approach enforces maximal monotonicity of the approximants by construction.

Unlike classical universality results based on global uniform convergence, the graph convergence framework considered here is inherently local. 
More precisely, the approximating encoder--decoder architectures depend on the compact set on which the graph approximation is evaluated.

\subsection{Contribution}

A summary of our contributions is as follows : 
\begin{enumerate}[(1)]

\item \textbf{A paradigm shift in operator approximation.}
We show, through explicit counterexamples, that classical approximation frameworks in operator learning -- based on uniform or $L^p$ convergence -- are fundamentally inadequate for closed operators (Examples~\ref{ex:uniform-conv} and~\ref{ex:Lp-conv}). 
These results indicate that the limitation lies not in the expressive power of models, but rather in the topology used to measure approximation. 
This suggests that the notion of approximation itself must be reconsidered.

\item \textbf{Graph convergence as the appropriate topology.}
We identify graph convergence as a natural notion of approximation for closed operators, since graph convergence measures discrepancies directly at the level of graphs rather than operator values.
More precisely, we introduce a local graph distance (Definition~\ref{def:graph-distance}) that quantitatively measures discrepancies between closed graphs on compact subsets. We further prove that convergence of this local graph distance
to zero on every compact subset is equivalent to graph convergence
(Proposition~\ref{prop:equiv-graph}).

\item \textbf{Approximation of maximally monotone operators in local graph convergence.}
We show that monotonicity and closedness alone are insufficient for approximability even in graph convergence
(Example~\ref{ex:graph-conv}).
In contrast, within the class of maximally monotone operators, although approximation in uniform or $L^p$ convergence still fails in general, approximation in the sense of local graph convergence becomes possible.
In particular, we prove that any maximally monotone operator can be approximated by continuous encoder--decoder architectures
(Theorem~\ref{thm-main-1}), while maximal monotonicity can also be preserved in the approximation process (Corollary~\ref{cor-main-2}).

\end{enumerate}
%
%

\subsection{Notation}

Let $H$ be a Hilbert space.
We write $A : D(A) \subset H \to H$ to denote an operator $A$ defined on the domain $D(A)$ with values in $H$.
When we write $A : H \to H$, we mean that $D(A) = H$.
We denote by $R(A)$ the range of $A$.
If $A$ is a set-valued operator we write $A : D(A) \subset H \to 2^H$.
In this case, we define the graph of $A$ by
$
\mathrm{gph}(A) := \{ (u,v) \in H \times H \mid u \in D(A),\; v \in A(u) \}.
$
Finally, we denote the Euclidean norm on $\mathbb{R}^n$ by $|\cdot|$.

\section{Non-Approximability of Closed Operators}

In this section we define closed operators and several notions of convergence, and highlight the difficulties in establishing universal approximation results for closed operators in these topologies.

\subsection{Closed operator}

\begin{definition}[Closed operator]
We say that $A : D(A) \subset H \to 2^H$ is \emph{closed} if for any sequences
\[
x_m \in D(A), \quad x_m \to x \ \text{in } H, \qquad
y_m \in A(x_m), \quad y_m \to y \ \text{in } H,
\]
it follows that
$x \in D(A)$ and $y \in A(x)$. In other words, $\mathrm{gph}(A)$ is closed in $H \times H$. 
\end{definition}

\begin{example}
\label{ex:closed}
Closed operators include differential operators and, more generally, inverses of compact operators, which naturally arise in scientific computing and inverse problems. 

\begin{enumerate}[(a)]
  \item The differential operator $A : H^1(0,1) \subset L^2(0,1) \to L^2(0,1)$, $A(u)(x) = \frac{d}{dx}u(x)$. 
  \item Laplacian operator $\Delta : H^2(\mathbb{R}^d) \subset L^2(\mathbb{R}^d) \to L^2(\mathbb{R}^d)$.
  \item The inverse of compact operator $A : R(K) \subset H \to 2^H$, $A= K^{-1}$ where $K:H \to H$ is compact. 
\end{enumerate}

\end{example}

\subsection{Locally uniform convergence}

\begin{definition}[Locally uniform convergence]
\label{def:uniform-conv}
Let $A:D(A)\subset H \to H$ and $A_m: D(A_m) \subset H \to H$. We say that $A_m$ converges to $A$ in locally uniform convergence and write 
$$
A_m\xrightarrow{\text{\rm loc. unif.}}A
$$ 
if
for any compact set $K \subset H$ with $K \cap D(A) \cap D(A_m) \neq \emptyset$ for all $m$, it holds that 
\[
\sup_{x \in K \cap D(A) \cap D(A_m)} \|A_m(x) - A(x)\|_H \to 0. 
\]
\end{definition}

Closed operators are, in general, not continuous and therefore cannot be obtained as a limit of continuous operators under locally uniform convergence (recall that a uniform limit of continuous operators is itself continuous). We illustrate this with the following simple example

\begin{example}
\label{ex:uniform-conv}
Define the differential operator $A: H^1(0,1) \subset L^2(0,1) \to L^2(0,1)$ by 
\[
A(u)(x) := \frac{d}{dx} u(x).
\]
Then, $A$ is a closed operator (Example~\ref{ex:closed} (a)). 
Define the compact set 
\[
K:=\{0\}\cup\{v_n:n\in\mathbb{N}\}\subset D(A), \qquad
v_n(x):=\frac{1}{n}\sin(n\pi x)\in H^1(0,1).
\]
Then, for any $\varepsilon \in (0,1)$, there is no continuous operator $A_\varepsilon: D(A_\varepsilon) \subset L^2(0,1)\to L^2(0,1)$ with $K \subset D(A_\varepsilon)$ such that
\[
\sup_{u\in K}\|A_\varepsilon(u)-A(u)\|_{L^2(0,1)} \leq \varepsilon.
\]
\end{example}
The proof can be found in Appendix~\ref{app:ex:uniform-conv}.

\subsection{$L^p$ convergence}

\begin{definition}[$L^p$ convergence]
Let $\pi$ be a probability measure on $H$. 
Let $1 \le p < \infty$. Let $A:D(A)\subset H \to H$ and $A_m: D(A_m) \subset H \to H$ such that $\mathrm{supp}(\pi) \subset D(A)$ and $\mathrm{supp}(\pi) \subset D(A_m)$ for all $m$.
We say that $A_m$ converges to $A$ in $L^p(\pi;H)$ and write
\[
A_m\xrightarrow{L^p(\pi;H)}A
\quad \text{ if } \quad 
\int \|A_m(u)-A(u)\|_H^p \, d\pi(u) \to 0.
\]
\end{definition}

\begin{proposition}
\label{prop:loc-uni-Lp}
Let \(\pi\) be a probability measure on \(H\) such that
\(\operatorname{supp}(\pi)\) is compact. Assume that $\mathrm{supp}(\pi) \subset D(A)$ and $\mathrm{supp}(\pi) \subset D(A_m)$ for all $m$.
Then it holds that
$$
A_m\xrightarrow{\text{\rm loc. unif.}}A \ \implies \ A_m\xrightarrow{L^p(\pi;H)}A.
$$
\end{proposition}
The proof can be found in Appendix~\ref{app:prop:loc-uni-Lp}.

The next example shows that, although $L^p(\pi;H)$-convergence is weaker than locally uniform conver\-gence, in general a closed operator cannot be approximated by continuous ones in the $L^p(\pi;H)$ sense.\looseness=-1

\begin{example}
\label{ex:Lp-conv}
Define the differential operator
$
A: H^1_{\mathrm{per}}(0,1) \subset L^2_{\mathrm{per}}(0,1) \to L^2_{\mathrm{per}}(0,1),
$ by 
\[
A(u)(x) := \frac{d}{dx} u(x).
\]
Then, $A$ is a closed operator. Moreover, $A$ is maximally monotone (See Definition~\ref{def:monotone}). 
Define a probability measure
\[
\pi:=\sum_{n=2}^\infty w_n\,\delta_{u_n},
\quad
w_n:=\frac{c}{n^{1+p/2}}, \quad 
u_n(x):=\frac{1}{\sqrt n}\sin(2\pi n x)
\]
where $c>0$ is chosen so that $\sum_{n=2}^\infty w_n=1$.
Then, for any $\varepsilon \in (0,1)$, there is no continuous $A_\varepsilon: D(A_\varepsilon) \subset L^2_{\mathrm{per}}(0,1) \to L^2_{\mathrm{per}}(0,1)$ with $\mathrm{supp}(\pi) \subset D(A_\varepsilon)$ such that
\[
\int \|A_\varepsilon(u)-A(u)\|_{L^2(0,1)}^p d\pi(u) \leq \varepsilon.
\]
\end{example}
The proof can be found in Appendix~\ref{app:ex:Lp-conv}.

\subsection{Graph convergence}

Now we introduce the graph (Painlev\'e-Kuratowski) convergence, see \cite[Sec. 2.2]{adly2022preservation}.

\begin{definition}[Graph convergence]
\label{def:Graph-convergence}
Let $A: D(A) \subset H\to 2^H$ and $A_m: D(A_m) \subset H\to 2^H$. We say that $A_m$ converges to $A$ in the sense of \emph{graph convergence} and write 
$$
A_m\xrightarrow{g}A
$$ 
\[
\text{if }
\begin{cases}
\text{\rm (outer limit condition) }
(x_m,y_m)\in \mathrm{gph}(A_m),\ (x_m,y_m)\to(x,y) \ \Rightarrow\ (x,y)\in \mathrm{gph}(A)\\[2pt]
\text{\rm (inner limit condition) }
\forall (x,y)\in \mathrm{gph}(A)\ \exists (x_m,y_m)\in \mathrm{gph}(A_m)
\text{ such that } (x_m,y_m)\to(x,y).
\end{cases}
\]
Equivalently, the outer limit condition may be stated
for every subsequence $(m_k)_{k\in\mathbb N}$.
\end{definition}

Graph convergence intuitively means that $\mathrm{gph}(A_m)$ converge to $\mathrm{gph}(A)$ as subsets of $H \times H$.
While the locally uniform and $L^p(\pi;H)$ convergences are defined for single-valued operators, graph convergence makes sense for set-valued operators. The following proposition shows that when operators are single-valued, the graph convergence is weaker than the locally uniform convergence. 

\begin{proposition}
\label{prop:loc-uni-graph}
Let $A: D(A) \subset H\to H$ and $A_m: D(A_m) \subset H\to H$ be single-valued operators. Assume that $D(A) = D(A_m)$, and $A$ is closed. 
Then it holds that
$$
A_m\xrightarrow{\text{\rm loc. unif.}}A \ \implies \ A_m\xrightarrow{g}A.
$$
\end{proposition}
The proof can be found in Appendix~\ref{app:prop:loc-uni-graph}.

Just as with $L^p(\pi;H)$-convergence, 
even though graph convergence is weaker than locally uniform one, an arbitrary closed operator cannot be approximated by continuous ones in the graph sense. 

\begin{example}
\label{ex:graph-conv}
Let \(H=\mathbb{R}\). Define a set-valued operator
\(A: \mathbb{R}\to 2^{\mathbb{R}}\) by \(D(A)=\mathbb{R}\) and
\[
A(x):=
\begin{cases}
\{0\}, & x<0,\\[2mm]
\{0,1\}, & x=0,\\[2mm]
\{1\}, & x>0.
\end{cases}
\]
Then \(A\) is closed. Moreover, $A$ is monotone, but not maximally monotone (See Definition~\ref{def:monotone}).
In addition, there is no sequence of continuous operators
\(A_m:\mathbb{R}\to\mathbb{R}\) such that
\[
A_m\xrightarrow{g} A .
\]
\end{example}

The proof can be found in ~\ref{app:ex:graph-conv}.


Now we define a notion of distance that will play an important role in our approximation results.

\begin{definition}[Local graph distance]\label{def:graph-distance}
For $A: D(A) \subset H\to 2^H$ and $B: D(B) \subset H\to 2^H$, define the \emph{local graph distance} for any compact set $K \subset H \times H$
\begin{align*}
    d_{\mathrm{gph},K}(A_m,A)
& :=
\max\left\{ h_{K}(\mathrm{gph}(B),\mathrm{gph}(A)), h_{K}(\mathrm{gph}(A),\mathrm{gph}(B)) \right\},
\end{align*}
where $h_{K}$ is defined by for $X, Y \subset H \times H$
\[
h_K(X,Y) :=
\begin{cases}
\sup_{(x, y) \in X \cap K} \inf_{(u,v) \in Y} \|(x,y) - (u,v) \|_{H\times H}, & X \cap K \neq \emptyset,\\[2mm]
0, & X \cap K = \emptyset.
\end{cases}
\]
\end{definition}



The reason for introducing the local graph distance is that graph convergence, as defined in Definition~\ref{def:Graph-convergence}, is qualitative in nature.
In contrast, the local graph distance provides a quantitative measure of discrepancies between graphs on compact subsets, which allows us to formulate quantitative approximation results.

In the finite-dimensional setting, \cite[Remark 2.1]{adly2022preservation}
shows that graph convergence coincides with the $\rho$-Hausdorff
convergence, which is induced by a genuine metric topology.
In infinite-dimensional Hilbert spaces, however, extending such
a metric characterization is nontrivial.
For this reason, we introduce a weaker local graph distance,
which is generally not a metric.
Nevertheless, the following proposition shows that it still characterizes graph convergence equivalently.


\begin{proposition}
\label{prop:equiv-graph}
Let $A: D(A) \subset H\to 2^H$ and $A_m: D(A_m) \subset H\to 2^H$ be closed. Then the following are equivalent :
\[
\begin{aligned}
A_m \xrightarrow{g} A
\quad \Longleftrightarrow \quad & 
d_{\mathrm{gph},K}(A_m,A)\to 0 \ \text{ for every compact } K \subset H\times H.
\end{aligned}
\]
\end{proposition}
We refer to Appendix~\ref{app:prop:equiv-graph} for the proof.

\section{Approximation of Maximally Monotone Operators}
\label{sec:approximation}

In the previous section, we discussed the difficulty of approximating closed operators in the sense of uniform, $L^p$, and even graph convergence. 
A natural solution is to impose additional structure on the target operator. 
In this section, we focus on maximally monotone operators and establish their universal approximation.

\subsection{Maximally monotone operators}
\label{sec:maximally-monotone}

\begin{definition}[Monotone and maximally monotone operators~\cite{brezis1973ope}]
\label{def:monotone}
We say that $A : D(A) \subset H \to 2^H$ is \emph{monotone} if
\[
  \langle y - v,\, x - u \rangle_H \ge 0
  \quad \forall (x,y), (u,v)\in \mathrm{gph}(A).
\]
We say that monotone operator $A : H \to 2^H$ is \emph{maximally monotone} if there exists no monotone operator $B : D(B) \subset H \to 2^H$ such that $\mathrm{gph}(B)$ properly contains $\mathrm{gph}(A)$. 
\end{definition}

It is well known that any maximally monotone operator is closed (e.g., {\cite[Prop. 20.33]{bauschke2011convex}}).

\begin{proposition}
\label{prop:maxmally-closed}
If $A : D(A) \subset H \to 2^H$ is maximally monotone, then $A$ is closed. 
\end{proposition}

\begin{example}
\label{ex-max-mono}
We give some examples of maximally monotone operators. Note that appropriate boundary conditions may be required to ensure maximal monotonicity. 

\begin{enumerate}[(a)]
    \item Subdifferentials of convex functionals $H \to \RI$. For example, the $p$-Laplacian, $1 < p < +\infty$,  $A : W^{1,p}(\Omega) \subset L^2(\Omega) \to L^2(\Omega)$ for a bounded open set $\Omega \subset \mathbb{R}^d$, defined as $A(u) = -\mathrm{div}(|\nabla u|^{p-2} \nabla u)$, which is a subgradient of the $p$-Dirichlet energy functional $J \colon L^2(\Omega) \to \RI$ given by
    \begin{equation*}
        J(u) =
        \begin{cases}
        \frac1p\int_\Omega \norm{\grad u}^p \, dx,  & u \in W^{1,p}(\Omega), \\
        +\infty,  & u \in L^2(\Omega) \setminus W^{1,p}(\Omega).
        \end{cases}
    \end{equation*}
    \item The differentiation operator (the inverse Volterra operator~\cite{bauschke2010examples}) on periodic functions $A : H^1_{\mathrm{per}}(0,1) \subset L^2_{\mathrm{per}}(0,1) \to L^2_{\mathrm{per}}(0,1)$, $A(u)(x) = \frac{d}{dx}u(x)$ (which is not symmetric and therefore cannot be the subdifferential of a convex function).
  %
\end{enumerate}
More examples can be found in Appendix~\ref{app:ex-max-mono} and in~\cite{benedetti2024differential}.
\end{example}

\begin{remark}
We give several remarks on the approximability of maximally monotone operators with respect to the choice of convergence topology.

\begin{enumerate}[(1)]
\item 
Example~\ref{ex-max-mono} (b) shows that the differential operator 
$
A: H^1_{\mathrm{per}}(0,1) \subset L^2_{\mathrm{per}}(0,1) \to L^2_{\mathrm{per}}(0,1)$, $A(u)(x) := \frac{d}{dx}u(x)$ is maximally monotone. 
On the other hand, Example~\ref{ex:Lp-conv} shows that $A = \frac{d}{dx}$ cannot be approximated by any sequence of continuous operators in $L^p$, which also rules out uniform convergence by Proposition~\ref{prop:loc-uni-Lp}. 
This demonstrates that, even under maximal monotonicity, approximation in $L^p$ or uniform convergence may fail.

\item
Example~\ref{ex:graph-conv} shows that there exist closed and monotone but not maximally monotone operators that cannot be approximated by any sequence of continuous operators even in the graph  sense. 
This indicates that maximal monotonicity is a sharp structural condition for approximability.
\end{enumerate}
\end{remark}

\begin{definition}[Yosida approximation]
The \emph{Yosida approximation} of maximally monotone operator $A : D(A) \subset H \to 2^H$ is a family $(A_\lambda)_{\lambda>0}$ of single-valued operator $:H \to H$ defined by
\[
A_\lambda := \frac{I - J_\lambda}{\lambda}, \quad
J_\lambda := (I + \lambda A)^{-1}. 
\]
\end{definition}


By Minty's theorem \cite[Thm. 21.1]{bauschke2011convex}, the Yosida approximation $(A_\lambda)_{\lambda>0}$, and in particular the resolvent $J_\lambda$, is well-posed. Moreover, the Yosida approximation $(A_\lambda)_{\lambda>0}$ has the following properties  (e.g., \cite[Cor. 23.10, Cor. 23.46]{bauschke2011convex}).
\begin{proposition}
Let $A : D(A) \subset H \to 2^H$ be maximally monotone.
$A_\lambda : H \to H$ is $\frac{1}{\lambda}$ Lipschitz continuous and maximally monotone. Moreover, it holds that for $x \in D(A)$, 
$
\lim_{\lambda \to 0} A_\lambda(x) = A^0(x),
$
where $A^0(x)$ denotes the element of minimal norm in $A(x)$.
\end{proposition}

Note that the above proposition states that, if $A$ is a single-valued maximally monotone operator, then $A^0$ coincides with $A$. In particular, this implies that $A_\lambda$ converges to $A$ pointwise. On the other hand, the following proposition shows that the Yosida approximation of a (possibly set-valued) maximally monotone operator converges to $A$ in the sense of graph convergence.

\begin{proposition}
\label{prop:yosida-graph-conv}
Let $A : D(A) \subset H \to 2^H$ be maximally monotone. Then it holds that 
\[
A_\lambda\xrightarrow{g}A \quad (\lambda \to 0)
\]
\end{proposition}

A proof of this statement in the finite-dimensional case $H = \R^n$ can be found in~\cite{rockafellar1998variational} but it holds for a general Hilbert space $H$, see Appendix~\ref{app:prop:maxmally-closed}.

\subsection{Encoder--decoder architectures}

\begin{definition}[Encoder--decoder approximation property]
\label{def:EDAP}
A Hilbert space $H$ is said to have the encoder--decoder approximation property (EDAP) if there exist sequences of 1-Lipschitz continuous maps 
\[
E_m : H \to \mathbb{R}^{w_m}, \qquad D_m : \mathbb{R}^{w_m} \to H,
\]
with $E_m(0)=0$, $D_m(0)=0$, and $(w_m)_{m \in \mathbb{N}} \subset \mathbb{N}$ with $w_m \to \infty$, such that for any compact set $K \subset H$, it holds that
\[
\lim_{m \to \infty}
\sup_{u \in K} \| u - D_m \circ E_m(u) \|_{H} = 0.
\]
\end{definition}

\begin{definition}[Universal approximator]
\label{def:UA}
A class $\mathcal{F}$ of continuous functions is called a universal approximator if, for any $p \in \mathbb{N}$, any compact set $K \subset \mathbb{R}^p$, and any continuous function $\varphi : \mathbb{R}^p \to \mathbb{R}^p$, there exists a sequence $(\varphi_n)_{n \in \mathbb{N}} \subset \mathcal{F}$ such that
\begin{equation}
\label{eq:uniform-def}
\lim_{n \to \infty}\sup_{x \in K}\left| \varphi_n(x) - \varphi(x) \right| = 0.
\end{equation}
Similarly, a class $\mathcal{F}^{\mathrm{Lip}}$ of 1-Lipschitz continuous functions is called a 1-Lipschitz universal approximator if, for any $p \in \mathbb{N}$, any compact set $K \subset \mathbb{R}^p$, and any 1-Lipschitz continuous function $\varphi : \mathbb{R}^p \to \mathbb{R}^p$, there exists a sequence $(\varphi_n)_{n \in \mathbb{N}} \subset \mathcal{F}^{\mathrm{Lip}}$ such that
 such that 
\eqref{eq:uniform-def} holds.
\end{definition}

\begin{definition}[Encoder--decoder architectures]
\label{def:Encoder--Decoder}
Let $H$ be a Hilbert space satisfying EDAP with encoder--decoder pairs $(E_m, D_m)$. 
Let $\mathcal{F}$ (resp.\ $\mathcal{F}^{\mathrm{Lip}}$) be a (1-Lipschitz) universal approximator.
We call operators of the form
\[
D_m \circ \varphi \circ E_m : H \to H, \qquad \varphi \in \mathcal{F} \ (\text{resp.\ } \mathcal{F}^{\mathrm{Lip}}),
\]
(1-Lipschitz) encoder--decoder architectures.
\end{definition}

The encoder--decoder architectures $D_m \circ \varphi \circ E_m$ constitute a canonical class of operator models; see, e.g., \cite{godeke2025new, kovachki2024operator}. 
For example, classes such as DeepONet and neural operators can be interpreted within this framework. 

A typical example for encoder-decoder pair $(E_m, D_m)$ is as follows: let $H$ be a Hilbert space with a complete orthonormal basis $(e_i)_{i \in \mathbb{N}}$. 
Define $E_m$ as the orthogonal projection onto the subspace $\mathrm{span}\{e_i : i \leq m\}$, and let $D_m$ be its adjoint operator. 
Then both $E_m$ and $D_m$ are 1-Lipschitz continuous.

The class $\mathcal{F}$ can be taken as a family of neural networks that are universal approximators for continuous functions on Euclidean space \cite{cybenko1989approximation,hanin2017approximating,hornik1989multilayer}. 
Similarly, the class $\mathcal{F}^{\mathrm{Lip}}$ can be chosen as neural networks that preserve $1$-Lipschitz continuity, which have been studied in recent works such as
\cite{anil2019sorting,miyato2018spectral,murari2025approximation,neumayer2023approximation}.

Recent work \cite[Thm. 4.4]{godeke2025new} shows that encoder-decoder architecture ensure the locally uniform convergence to any continuous operator. 
\begin{proposition}
\label{prop:EDA-uniform-conv}
Let $\mathcal{F}$ be universal approximator, and let $H$ have the EDAP with $E_m$ and $D_m$. 
Let $A : D(A) \subset H \to H$ be a continuous operator. 
Then, there exist a sequence $(\varphi_m)_{m \in \mathbb{N}} \subset \mathcal{F}$ such that
\begin{equation}
\label{eq:EDA-uniform-conv}
D_m \circ \varphi_m \circ E_m \xrightarrow{\text{\rm loc. unif.}}A, \quad (m \to \infty).
\end{equation}
Moreover, if $A : D(A) \subset H \to H$ is 1-Lipschitz continuous, and $\mathcal{F}^{\mathrm{Lip}}$ is 1-Lipschitz universal approximator, then, there exist a sequence $(\varphi_m)_{m \in \mathbb{N}} \subset \mathcal{F}^{\mathrm{Lip}}$ such that encoder--decoder architectures $E_m \circ \varphi_m \circ D_m$ is 1-Lipschitz continuous, and \eqref{eq:EDA-uniform-conv} holds.
\end{proposition}

\subsection{Approximation result}

We first establish the following main result that 

\begin{theorem}
\label{thm-main-1}
Let $\mathcal{F}$ be a universal approximators, and assume that $H$ has the EDAP with encoder and decoder mappings $E_m$ and $D_m$. 
Let $A : D(A) \subset H \to 2^H$ be maximally monotone. 
Then, 
for any compact set $K \subset H \times H$, there exists a sequence $(\varphi_{m})_{m\in\mathbb{N}} \subset \mathcal{F}$ such that as $m \to \infty$
\[
d_{\mathrm{gph},K}(A_m,A)\to 0, \quad
A_{m} := D_m \circ \varphi_{m} \circ E_m.
\]
\end{theorem}

This theorem states that any maximally monotone operator can be approximated by continuous encoder--decoder architectures in the sense of local graph convergence. 
The proof is based on using the Yosida approximation $A_\lambda$ (see Proposition~\ref{prop:yosida-graph-conv}), which is approximated by encoder--decoder architectures via Proposition~\ref{prop:EDA-uniform-conv}. 
We refer to Appendix~\ref{app:thm-main-1} for details. 
Moreover, using quantitative approximation results for neural networks applied to the Yosida approximation, we can obtain a quantitative error estimates. See Appendix~\ref{app:cor:main-quant}  for details.

\medskip

While the above encoder--decoder architectures do not, in general, preserve maximal monotonicity, 
we can establish the following structure-preserving approximation result inspired by the resolvent formulation of the Yosida approximation.

\begin{corollary}
\label{cor-main-2}
Let $\mathcal{F}^{\mathrm{Lip}}$ be a class of $1$-Lipschitz universal approximators, and assume that $H$ has the EDAP with $E_m$ and $D_m$. 
Let $A : D(A) \subset H \to 2^H$ be maximally monotone. Let $(\lambda_m)_{m\in\mathbb{N}} \subset \mathbb{R}_{+}$ with $\lambda_m \to 0$. 
Then, for any compact set $K \subset H \times H$, there exists sequences $(\psi_m)_{m\in\mathbb{N}} \subset \mathcal{F}^{\mathrm{Lip}}$ such that
\[
d_{\mathrm{gph},K}(B_m,A)\to0,
\quad 
B_m := \frac{I - D_m \circ \psi_m \circ E_m}{2\lambda_m}.
\]
Here, $D_m \circ \psi_m \circ E_m$ is a $1$-Lipschitz encoder--decoder architecture, and $B_m : H \to H$ is maximally monotone. 
\end{corollary}

The key idea is based on the resolvent representation of the Yosida approximation
$A_\lambda = \frac{I - (2J_\lambda - I)}{2\lambda}$. Here, $2J_\lambda - I$ is non-expansive (i.e., $1$-Lipschitz). 
We approximate this mapping by $1$-Lipschitz encoder--decoder architectures, which yields a sequence of maximally monotone operators by construction. We refer to Appendix~\ref{app:cor-main-2} for details.


\section{Experiments}
\label{sec:experiments}

We study operator learning for the following two differential operators:
\begin{align*}
& A_1 : u(t) \mapsto u'(t),
\qquad t \in [0,1],
\\
&
A_2 : u(x,y) \mapsto 
-\operatorname{div}\bigl(|\nabla u(x,y)|^{p-2}\nabla u(x,y)\bigr),
\qquad (x,y) \in [0,1]^2.
\end{align*}
Here, \(A_1\) is a prototypical unbounded differential operator,
while \(A_2\) is a nonlinear monotone operator of \(p\)-Laplacian type.
We use synthetic Fourier-series data generated according to
\eqref{eq:ex-generate-1} and~\eqref{eq:ex-generate-2},
inspired by Examples~\ref{ex:uniform-conv}
and~\ref{ex:Lp-conv}.
These examples show that differential operators cannot, in general,
be approximated under uniform or \(L^p\) convergence.
In particular, high-frequency inputs are mapped to outputs with
significantly amplified amplitudes, reflecting the unbounded nature
of differential operators
(see Figures~\ref{fig:high-frequency-derivative}
and~\ref{fig:high-frequency-plaplacian}).

We compare four FNO-based models trained with different objectives:
$\ell^2$ loss, soft $\ell^\infty$ loss, soft graph loss,
and a structured model based on the resolvent parametrization
$
\hat{A}(u)=\frac{u-\hat{S}(u)}{\lambda},
$
together with non-expansive regularization inspired by
Corollary~\ref{cor-main-2}.
See \eqref{eq:ell2-loss}--\eqref{eq:structured-graph}.

We evaluate the models using the following metrics:
Test MSE, Mean Rel-L2, Worst Rel-L2, Test Graph,
Mono mean-viol, Mono worst, and Mono frac;
see \eqref{eq:test-mse}--\eqref{eq:mono-frac}.
The metrics Mono mean-viol, Mono worst, and Mono frac quantify
how often and how severely the monotonicity condition is violated,
and values close to zero indicate better preservation of monotonicity.
A perfectly monotone operator yields zero mean violation,
zero worst violation, and zero violation fraction.

See Appendix~\ref{app:details-experiment}
for full details of the experimental setup.

The results are summarized in Tables~\ref{tab:results-derivative-main} and ~\ref{tab:results-plap-main} where values are reported as mean over four random seeds (Tables~\ref{tab:results-derivative} and ~\ref{tab:results-plap} includes mean $\pm$ standard deviation).

\begin{table}[h]
\centering
\footnotesize
\setlength{\tabcolsep}{4pt}

\caption{Comparison of test evaluations for \(A_1(u)=u'\).}
\label{tab:results-derivative-main}

\begin{tabular}{lccccccc}
\toprule
Model
& Test MSE.
& Mean Rel-$\ell^2$.
& Worst Rel-$\ell^2$.
& Test Graph.
& Mono mean-viol.
& Mono worst.
& Mono frac.
\\
\midrule

$\ell^2$
& $1.1230$
& $0.0159$
& $0.0389$
& $7.7351$
& $20.5333$
& $-246.1772$
& $0.6375$
\\

$\ell^\infty$
& $2.3120$
& $0.0263$
& $0.0956$
& $14.4008$
& $43.5230$
& $-493.7024$
& $0.7730$
\\

graph
& $\mathbf{0.0421}$
& $\mathbf{3.368\times10^{-3}}$
& $\mathbf{0.0116}$
& $\mathbf{0.2790}$
& $1.8287$
& $-33.9634$
& $0.5284$
\\

graph + structure
& $0.1784$
& $7.934\times10^{-3}$
& $0.0170$
& $0.8975$
& $\mathbf{0.2506}$
& $\mathbf{-36.7055}$
& $\mathbf{0.0111}$
\\

\bottomrule
\end{tabular}
\end{table}

\begin{table}[h]
\centering
\footnotesize
\setlength{\tabcolsep}{4pt}

\caption{
Comparison of test evaluations for
\(
A_2(u)=
-\operatorname{div}(|\nabla u|^{p-2}\nabla u)
\).
}
\label{tab:results-plap-main}

\begin{tabular}{lccccccc}
\toprule
Model
& Test MSE.
& Mean Rel-$\ell^2$.
& Worst Rel-$\ell^2$.
& Test Graph.
& Mono mean-viol.
& Mono worst.
& Mono frac.
\\
\midrule

$\ell^2$
& $1.93\times10^{4}$
& $6.13\times10^{-1}$
& $6.88\times10^{-1}$
& $1.59\times10^{4}$
& $0.000$
& $0.000$
& $0.000$
\\

$\ell^\infty$
& $2.08\times10^{4}$
& $6.37\times10^{-1}$
& $7.23\times10^{-1}$
& $1.69\times10^{4}$
& $0.000$
& $0.000$
& $0.000$
\\

graph
& $\mathbf{1.51\times10^{4}}$
& $\mathbf{5.43\times10^{-1}}$
& $\mathbf{6.14\times10^{-1}}$
& $\mathbf{1.28\times10^{4}}$
& $0.000$
& $0.000$
& $0.000$
\\

graph + structure
& $3.83\times10^{4}$
& $8.71\times10^{-1}$
& $9.15\times10^{-1}$
& $2.46\times10^{4}$
& $0.000$
& $0.000$
& $0.000$
\\

\bottomrule
\end{tabular}
\end{table}

In both Tables~\ref{tab:results-derivative-main} and ~\ref{tab:results-plap-main}, the graph model achieves the best accuracy in Test MSE, Mean Rel-L2, Worst Rel-L2, and Test Graph.
In Table~\ref{tab:results-plap-main}, all models exhibit essentially zero monotonicity violations, indicating that the generated data already lies in a strongly monotone regime. On the other hands, in Tables~\ref{tab:results-derivative-main}, while $\ell^2$ and $\ell^\infty$ models exhibiting monotonicity violations, graph and graph + structure models reduce Mono mean-viol, Mono worst, and Mono frac. 
Meanwhile, the graph + structure model again demonstrates a trade-off between strict structure preservation and approximation accuracy, suggesting that strong nonexpansive regularization may restrict model expressivity.

\section{Conclusion}
\label{sec:conclusion}

We studied the universal approximation of maximally monotone operators from the viewpoint of graph convergence, rather than classical uniform or $L^p$ convergence. 

The limitations of our work and possible future directions are summarized as follows:

\begin{enumerate}[(a)]
    \item Theorem~\ref{thm-main-1} states that any maximally monotone operator can be approximated by continuous encoder--decoder architectures in the sense of local graph convergence. Note, however, that the encoder--decoder architectures constructed in the theorem depend on the compact set. Therefore, the theorem does not guarantee global graph convergence; see Proposition~\ref{prop:equiv-graph}. This is one limitation of our approach.
    \item Although Corollary~\ref{cor:main-quant} in Appendix~\ref{app:cor:main-quant} provides an upper bound on the local graph distance between the target operator and the encoder--decoder architecture, the bound is not fully explicit and is not uniform over compact sets (see Remark~\ref{rem:quantitative-cor}). Deriving sharper and more explicit quantitative estimates remains an important direction for future work.
    \item A limitation of the proposed graph-based loss is its computational cost. While standard losses such as $\ell^2$ or soft $\ell^\infty$ scale linearly as $\mathcal{O}(N)$ with respect to the number $N$ of samples, the soft graph loss requires pairwise comparisons across samples, resulting in quadratic complexity $\mathcal{O}(N^2)$. This quadratic scaling may limit practical scalability for high-resolution discretizations or large datasets. Developing computationally efficient approximations of graph-based distances remains an important direction for future work.
    \item While our theoretical results are established for set-valued operators, our experiments focus on the single-valued case. Extending the experimental study to genuinely set-valued operators is an important direction for future work. In such settings, classical notions such as uniform or $L^p$ approximation are not even well-defined, highlighting the essential role of graph-distance-based formulations. We believe that the proposed graph-based framework provides a natural approach for learning such operators.
\end{enumerate}

\section*{Acknowledgments}
TF and TY are supported by JSPS KAKENHI Grant Number JP24K16949, 25H01453, 25K15148, JST CREST JPMJCR24Q5, JST ASPIRE JPMJAP2329.
TF and YK acknowledge the support of the Royal Society (International Collaboration Award ICA{\textbackslash}R2{\textbackslash}252033). This work is also supported by the Horizon Europe, MSCA-SE project 101131557 (REMODEL.)




\bibliographystyle{plain}
\bibliography{ref}

\newpage

\appendix

\section{Proofs}

\subsection{Proof of Example~\ref{ex:uniform-conv}}
\label{app:ex:uniform-conv}
\begin{proof}
Suppose, for contradiction, there is a continuous $A_\varepsilon: D(A_\varepsilon) \subset H\to H$ with $K \subset D(A_\varepsilon)$ 
$$
\sup_{u\in K}\|A_\varepsilon(u)-A(u)\|_{L^2(0,1)} \leq \varepsilon.
$$
In particular, 
\[
\|A_\varepsilon(0)\|_{L^2(0,1)}
= \|A_\varepsilon(0) - A(0)\|_{L^2(0,1)} \leq \varepsilon.
\]
We observe that since $\|A v_n\|_{L^2(0,1)}=\frac{\pi}{\sqrt{2}}$
\begin{align*}
\sup_{u\in K}\|A_\varepsilon(u)-A(u)\|_{L^2(0,1)} 
& \geq 
\|A_\varepsilon(v_n)-A(v_n)\|_{L^2(0,1)} 
\\
&
\geq \| A(v_n) \|_{L^2(0,1)} - \|A_\varepsilon(v_n)\|_{L^2(0,1)}
= \frac{\pi}{\sqrt{2}} - \|A_\varepsilon(v_n)\|_{L^2(0,1)}.
\end{align*}
Taking the limit as $n \to \infty$, using $v_n \to 0$ in $L^2(0,1)$, and using the continuity of $A_\varepsilon$ 
\[
\varepsilon \geq \sup_{u\in K}\|A_\varepsilon(u)-A(u)\|_{L^2(0,1)} 
\geq \frac{\pi}{\sqrt{2}} - \|A_\varepsilon(0)\|_{L^2(0,1)}
\geq \frac{\pi}{\sqrt{2}} - 1 > 0
\]
This contradicts with the fact that $\varepsilon \in (0,1)$ can be arbitrary small. 
\end{proof}

\subsection{Proof of Proposition~\ref{prop:loc-uni-Lp}}
\label{app:prop:loc-uni-Lp}
\begin{proof}
Let \(K:=\mathrm{supp}(\pi)\). Since \(K\) is compact and
\(A_m \xrightarrow{\mathrm{loc.\ unif.}} A\), we have
\[
\sup_{u\in K}\|A_m(u)-A(u)\|_H \to 0 .
\]
Then
\[
\int_H \|A_m(u)-A(u)\|_H^p\,d\pi(u)
=
\int_K \|A_m(u)-A(u)\|_H^p\,d\pi(u),
\]
because \(\pi\) is supported on \(K\). Hence
\[
\int_H \|A_m(u)-A(u)\|_H^p\,d\pi(u)
\le
\left(\sup_{u\in K}\|A_m(u)-A(u)\|_H\right)^p \pi(K).
\]
Since \(\pi\) is a probability measure and \(\pi(K)=1\), we obtain
\[
\int_H \|A_m(u)-A(u)\|_H^p\,d\pi(u)
\le
\left(\sup_{u\in K}\|A_m(u)-A(u)\|_H\right)^p
\to 0.
\]
Therefore,
\[
A_m\xrightarrow{L^p(\pi;H)}A.
\]
\end{proof}

\subsection{Proof of Example~\ref{ex:Lp-conv}}
\label{app:ex:Lp-conv}

\begin{proof}

We first show that \(A\) is maximally monotone.
Indeed, for any \(u\in D(A)\)
\[
\langle Au,u\rangle_{L^2}
=
\int_0^1 u'(x)u(x)\,dx
=
\frac12\int_0^1 (u(x)^2)'dx
=
\frac12\{u(1)^2-u(0)^2\}
=0.
\]
Hence \(A\) is monotone.

\medskip

We next show maximality by Minty's theorem. It suffices to prove that 
\[
\operatorname{Ran}(I+A)=H.
\]
Let \(f\in H\). We solve
\[
u+u'=f,\qquad u(0)=u(1).
\]
The general solution is
\[
u(x)=e^{-x}
\left(
c+\int_0^x e^t f(t)\,dt
\right),
\]
where \(c\in\mathbb{R}\) is a constant. The periodicity condition gives
\[
c
=
e^{-1}
\left(
c+\int_0^1 e^t f(t)\,dt
\right),
\]
and hence
\[
c
=
\frac{1}{e-1}\int_0^1 e^t f(t)\,dt.
\]
Therefore there exists \(u\in H^1_{\mathrm{per}}(0,1)\) such that
\[
(I+A)u=f.
\]
Thus
\[
\operatorname{Ran}(I+A)=H.
\]
By Minty's theorem, \(A\) is maximally monotone.

\medskip


We finally prove the non-approximability statement. 
Suppose, for contradiction, that there exists a continuous operator
\(A_\varepsilon:D(A_\varepsilon)\subset L^2_{\mathrm{per}}(0,1)\to L^2_{\mathrm{per}}(0,1)\) with 
\(\operatorname{supp}(\pi)\subset D(A_\varepsilon)\) such that
\[
\int \|A_\varepsilon(u)-A(u)\|_H^p\,d\pi(u)\le \varepsilon .
\]
Since \(u_n\to 0\) in \(H\), and since \(0\in \operatorname{supp}(\pi)\subset D(A_\varepsilon)\), 
the continuity of \(A_\varepsilon\) implies
\[
A_\varepsilon(u_n)\to A_\varepsilon(0) \quad \text{in } H.
\]
Hence there exist \(M>0\) and \(N\in\mathbb N\) such that
\[
\|A_\varepsilon(u_n)\|_H\le M \qquad (n\ge N).
\]
On the other hand,
\[
\|Au_n\|_H=\sqrt{2n}\pi.
\]
Therefore, for sufficiently large \(n\),
\[
\|A_\varepsilon(u_n)-Au_n\|_H
\ge \|Au_n\|_H-\|A_\varepsilon(u_n)\|_H
\ge \frac{\sqrt{2n}\pi}{2}.
\]
Consequently,
\[
\begin{aligned}
\int \|A_\varepsilon(u)-A(u)\|_H^p\,d\pi(u)
&=
\sum_{n=2}^\infty w_n
\|A_\varepsilon(u_n)-A(u_n)\|_H^p \\
&\ge
\sum_{n=N}^\infty 
\frac{c}{n^{1+p/2}}
\left(\frac{\sqrt{2n}\pi}{2}\right)^p \\
&=
c\left(\frac{\sqrt{2}\pi}{2}\right)^p
\sum_{n=N}^\infty \frac1n
=
\infty,
\end{aligned}
\]
which contradicts the assumed bound.

\end{proof}

\subsection{Proof of Proposition~\ref{prop:loc-uni-graph}}
\label{app:prop:loc-uni-graph}

\begin{proof}
We verify the outer and inner limit conditions.

\medskip
\noindent
\textbf{Outer limit condition.}
Assume that
\[
(x_m,A_m(x_m))\to (x,y)
\qquad\text{in } H\times H.
\]


Since $x_m\to x$, the set
\[
K:=\{x\}\cup \{x_m: m\in\mathbb N\}
\]
is compact in $H$. 
By the local uniform convergence on $K$ 
(note that $ D(A) = D(A_m)$),
\[
\sup_{z\in K\cap D(A)\cap D(A_m)}\|A_m(z)-A(z)\|_H \to 0.
\]
In particular,
\[
\|A_m(x_m)-A(x_m)\|_H \to 0.
\]
Since also $A_m(x_m)\to y$, it follows that
\[
A(x_m)\to y.
\]
Because $x_m\to x$ and $A$ is closed, we conclude that $x \in D(A)$ and $y=A(x)$. Thus the outer limit condition holds.

\medskip
\noindent
\textbf{Inner limit condition.}
Let $x\in D(A)$ be arbitrary. Set
\[
x_m:=x \qquad (m\in\mathbb N).
\]
Then $x_m\to x$.
Since $\{x\}\subset D(A) (= D(A_m))$ is compact in $H$, the assumed local uniform convergence yields
\[
\|A_m(x)-A(x)\|_H\to 0.
\]
Hence
\[
A_m(x_m)=A_m(x)\to A(x),
\]
and therefore
\[
(x_m,A_m(x_m))\to (x,A(x)).
\]
So the inner limit condition holds.

Combining the two parts, we obtain
\[
A_m\xrightarrow{g}A.
\]
\end{proof}

\subsection{Proof of Example~\ref{ex:graph-conv}}
\label{app:ex:graph-conv}

\begin{proof}
First, \(A\) is monotone. Indeed, if \(x<y\), then every element of
\(A(x)\) is less than or equal to every element of \(A(y)\). Hence, for
\((x,u),(y,v)\in \operatorname{gph}(A)\),
\[
(u-v)(x-y)\ge 0.
\]
Thus \(A\) is monotone.

The graph of \(A\) is
\[
\operatorname{gph}(A)
=
(-\infty,0)\times\{0\}
\cup \{(0,0),(0,1)\}
\cup (0,\infty)\times\{1\},
\]
which is closed in \(\mathbb{R}^2\). Hence \(A\) is closed.

It is not maximally monotone, since it is properly contained in the monotone
operator
\[
\widetilde A(x):=
\begin{cases}
\{0\}, & x<0,\\[2mm]
[0,1], & x=0,\\[2mm]
\{1\}, & x>0.
\end{cases}
\]

Now suppose, for contradiction, that there exists a sequence of continuous
maps \(A_m:\mathbb{R}\to\mathbb{R}\) such that
\[
A_m\xrightarrow{g} A.
\]
Since \((0,0),(0,1)\in \operatorname{gph}(A)\), the inner condition of graph
convergence gives sequences \(x_m,y_m\in\mathbb{R}\) such that
\[
x_m\to 0,\qquad A_m(x_m)\to 0,
\]
and
\[
y_m\to 0,\qquad A_m(y_m)\to 1.
\]
For sufficiently large \(m\), we have
\[
A_m(x_m)<\frac12<A_m(y_m).
\]
Since \(A_m\) is continuous, the intermediate
value theorem implies that there exists \(z_m\) between \(x_m\) and \(y_m\)
such that
\[
A_m(z_m)=\frac12.
\]
Since \(x_m,y_m\to 0\), we also have \(z_m\to 0\). Therefore
\[
(z_m,A_m(z_m))\to \left(0,\frac12\right).
\]
By the outer condition of graph convergence, this implies
\[
\left(0,\frac12\right)\in \operatorname{gph}(A).
\]
But \(A(0)=\{0,1\}\), so \(\frac12\notin A(0)\), a contradiction.
Thus no such continuous sequence exists.
\end{proof}

\subsection{Proof of Proposition~\ref{prop:equiv-graph}}
\label{app:prop:equiv-graph}

\begin{proof}

{\bf ($\Rightarrow$) Graph convergence implies local distance convergence.}

Assume that $A_m\xrightarrow{g}A$. Suppose, to the contrary, that there exist
a compact set $K\subset H\times H$, a number $\varepsilon>0$, and a subsequence,
still denoted by $m$, such that
\[
d_{\mathrm{gph},K}(A_m,A) \not\to 0.
\]
Since
\[
d_{\mathrm{gph},K}(A_m,A)\ge \varepsilon
\]
along a subsequence, after passing to a further subsequence if necessary,
one of the following two alternatives holds for all \(m\):
\[
(A)\quad
h_K(\operatorname{gph}(A),\operatorname{gph}(A_m))\ge \varepsilon,
\]
or
\[
(B)\quad
h_K(\operatorname{gph}(A_m),\operatorname{gph}(A))\ge \varepsilon.
\]

\medskip

{\bf Case (A): contradiction with the inner condition.}

Assume (A). Then for each $m$ there exists $(u_m,y_m) \in \mathrm{gph}(A)\cap K$ such that 
\begin{equation}
\label{eq:caseA-1}
\inf_{(v,z) \in \mathrm{gph}(A_m)}
\|(u_m,y_m) - (v,z) \|_{H \times H}
\ge \varepsilon/2.
\end{equation}
Since $K$ is compact, passing to a further subsequence, we may assume that 
\[
(u_m,y_m) \to (u,y).
\]
Since $\mathrm{gph}(A)$ is closed, we have $(u,y) \in \mathrm{gph}(A)$.
By the inner condition of graph convergence, there exists
$(v_m, z_m)\in \mathrm{gph}(A_m)$ such that
\[
(v_m, z_m)\to (u,y).
\]
Hence
\[
\begin{aligned}
\inf_{(v,z) \in \mathrm{gph}(A_m)}
 \|(u_m,y_m) - (v,z) \|_{H \times H}
&
\le \|(u_m, y_m)-(v_m, z_m)\|_{H\times H}
\\
&
\le \|(u_m, y_m)-(u,y)\|_{H\times H}+\|(v_m, z_m)-(u,y)\|_{H\times H}
\to 0,
\end{aligned}
\]
which contradicts \eqref{eq:caseA-1}.

\medskip

{\bf Case (B): contradiction with the outer condition.}

Assume (B). Then for each $m$ there exists $(v_m,z_m) \in \mathrm{gph}(A_m)\cap K$ such that
\[
\inf_{(u,y) \in \mathrm{gph}(A)}
\|(u,y) - (v_m,z_m) \|_{H \times H}
\ge \varepsilon/2.
\]

Since $(v_m,z_m)\in K$ and $K$ is compact, there exists a subsequence (not relabeled) such that
\[
(v_m,z_m) \to (v,z).
\]
By the outer limit condition, $(v,z) \in \mathrm{gph}(A)$.
On the other hand, by continuity of the distance,
\[
\inf_{(u,y) \in \mathrm{gph}(A)}
\|(u,y) - (v,z) \|_{H\times H}
=
\lim_{m\to\infty}
\inf_{(u,y) \in \mathrm{gph}(A)}
\|(u,y) - (v_m,z_m) \|_{H\times H}
\ge \varepsilon/2,
\]
which contradicts $(v,z) \in \mathrm{gph}(A)$.

\medskip

Thus $d_{\mathrm{gph},K}(A_m,A) \to 0$.

\bigskip

{\bf ($\Leftarrow$) Local distance convergence implies graph convergence.}

Assume that for every compact $K \subset H\times H$
\[
\max\left\{ h_{K}(\mathrm{gph}(A),\mathrm{gph}(A_m)), h_{K}(\mathrm{gph}(A_m),\mathrm{gph}(A)) \right\} \to 0.
\]

\medskip

{\bf (Outer limit condition)}

Assume $(x_m,y_m) \in \mathrm{gph}(A_m)$ and
\[
(x_m,y_m) \to (x,y).
\]

Define a compact set
\[
K := \{(x,y)\} \cup \{(x_m,y_m)\}_{m\in\mathbb N}.
\]
Hence
\[
\inf_{(u,z) \in \mathrm{gph}(A)}
\|(u,z) - (x_m,y_m)\|_{H\times H}
\leq h_{K}(\mathrm{gph}(A_m),\mathrm{gph}(A))
\to 0.
\]

Passing to the limit yields $(x,y) \in \mathrm{gph}(A)$.

\medskip

{\bf (Inner limit condition)}

Let $(x,y) \in \mathrm{gph}(A)$.
Define a compact 
$$
K =\{ (x,y) \}
$$
Then
\[
h_{K}(\mathrm{gph}(A),\mathrm{gph}(A_m)) \to 0,
\]
so
\[
\inf_{(v,z) \in \mathrm{gph}(A_m)}
\|(x,y)-(v,z)\|_{H\times H} \to 0.
\]

Thus we can choose $(x_m,y_m)\in \mathrm{gph}(A_m)$ such that
\[
(x_m,y_m) \to (x,y).
\]

\medskip

Therefore $A_m \xrightarrow{g} A$.

\end{proof}



\subsection{Proof of Proposition~\ref{prop:yosida-graph-conv}}
\label{app:prop:maxmally-closed}

\begin{proof}

\medskip

By the same computation of \cite[Lem. 12.14]{rockafellar1998variational}, the Yosida approximation 
\[
A_\lambda = (\lambda I + A^{-1})^{-1}.
\]
Since \(A^{-1}\) is maximally monotone, it is closed; see \cite[Prop. 20.22]{bauschke2011convex}, then by Lemma~\ref{lem:perturbation-graph}
\[
\lambda I + A^{-1} \xrightarrow{g}A^{-1}
\]
By Lemma~\ref{lem:inverse-graph} that is the continuity of the inverse operation with respect to graph convergence, 
\[
A_\lambda = (\lambda I + A^{-1})^{-1} \xrightarrow{g} (A^{-1})^{-1} = A.
\]
\end{proof}

\begin{lemma}\label{lem:perturbation-graph}
Let $A: D(A) \subset H \to 2^H$ be closed. Then, 
\[
\lambda_m I + A \xrightarrow{g} A.
\]
as $\lambda_m \to 0$.
\end{lemma}
\begin{proof}

\medskip\noindent
\textbf{Step 1: Outer limit condition.}
Assume that $x_m\to x$, $y_m\to y$ and $y_m\in(\lambda_m I + A)(x_m)$ for all $m$.
By the definition of $\lambda_m I + A$, we can choose $a_m\in A(x_m)$ such that
\[
y_m=\lambda_m x_m + a_m.
\]
Hence
\[
a_m = y_m - \lambda_m x_m \to y - 0 = y.
\]
Therefore $(x_m,a_m)\to(x,y)$ with $(x_m,a_m)\in\mathrm{gph}(A)$ for all $m$.
Since $A$ has a closed graph in $H\times H$,
and we conclude that $(x,y)\in\mathrm{gph}(A)$, i.e.\ $y\in A(x)$.

\medskip\noindent
\textbf{Step 2: Inner limit condition.}
Let $(x,y)\in\mathrm{gph}(A)$ be arbitrary. For each $m$, set
\[
x_m:=x,\qquad y_m:=\lambda_m x + y.
\]
Then clearly $x_m\to x$ and $y_m\to y$ as $m\to\infty$. Moreover, since $y\in A(x)$,
we have $y_m=\lambda_m x + y\in \lambda_m x + A(x) = (\lambda_m I + A)(x_m)$.
Thus $(x_m,y_m)\in\mathrm{gph}(\lambda_m I + A)$ for all $m$ and $(x_m,y_m)\to(x,y)$.

\end{proof}

\begin{lemma}\label{lem:inverse-graph}
Let $A_m : D(A_m) \subset H \to 2^H$ and $A \subset D(A) :  H \to 2^H$  such that $A_m \xrightarrow{g} A$. Then, 
\[
A_m^{-1} \xrightarrow{g} A^{-1}.
\]
\end{lemma}

\begin{proof}

\medskip\noindent
\textbf{Step 1: Outer limit condition for $A_m^{-1}$.}
Assume that $(y_m,x_m)\in\mathrm{gph}(A_m^{-1})$ for all $m$ and
\[
(y_m,x_m)\to(y,x)\quad\text{in }H\times H.
\]
By definition of the inverse graph, $(x_m,y_m)\in\mathrm{gph}(A_m)$ for all $m$ and
\[
(x_m,y_m)\to(x,y).
\]
Applying the outer condition in $A_m \xrightarrow{g} A$ yields $(x,y)\in\mathrm{gph}(A)$, i.e.\ $y\in A(x)$.
Equivalently, $x\in A^{-1}(y)$, namely $(y,x)\in\mathrm{gph}(A^{-1})$.

\medskip\noindent
\textbf{Step 2: Inner limit condition for $A_m^{-1}$.}
Let $(y,x)\in\mathrm{gph}(A^{-1})$ be arbitrary. Then $(x,y)\in\mathrm{gph}(A)$.
By the inner condition in $A_m \xrightarrow{g} A$, there exist $(x_m,y_m)\in\mathrm{gph}(A_m)$ such that
\[
(x_m,y_m)\to(x,y).
\]
Then $(y_m,x_m)\in\mathrm{gph}(A_m^{-1})$ for all $m$, and
\[
(y_m,x_m)\to(y,x).
\]
This proves the inner limit condition.

\end{proof}

\subsection{Proof of Theorem~\ref{thm-main-1}}
\label{app:thm-main-1}

\begin{proof}
Let $(\mu_n)_{n\in\mathbb N}\subset(0,\infty)$ be such that
$\mu_n\to0$. Let $K\subset H\times H$ be compact. Set
\[
K_0:=\pi_1(K),
\]
where $\pi_1:H\times H\to H$ is the first projection. 
Set
\[
K_1:=\overline{\{x+\mu_n y:(x,y)\in \operatorname{gph}(A)\cap K,\ n\in\mathbb N\}},
\qquad
K_*:=K_0\cup K_1 .
\]
Then \(K_*\) is compact.

For each fixed \(n\), since \(A_{\mu_n}:H\to H\) is continuous,
Proposition~\ref{prop:EDA-uniform-conv} yields 
\[
\inf_{\varphi\in\mathcal F}
\sup_{u\in K_*}
\|D_m\circ\varphi\circ E_m(u)-A_{\mu_n}(u)\|_H
\to0
\qquad (m\to\infty).
\]
Hence, by a diagonal argument, there exists a map \(n(m)\to\infty\)
and \(\varphi_m\in\mathcal F\) such that
\[
\sup_{u\in K_*}
\|D_m\circ\varphi_m\circ E_m(u)-A_{\mu_{n(m)}}(u)\|_H
\to0.
\]
Set \(\lambda_m:=\mu_{n(m)}\). Then \(\lambda_m\to0\), and
\begin{equation}
\label{eq:Am-Yosida-unif}
\sup_{u\in K_*}
\|A_m(u)-A_{\lambda_m}(u)\|_H\to0.
\end{equation}

We prove that
\[
d_{\mathrm{gph},K}(A_m,A)\to0.
\]

\medskip

First, we show
\[
h_K(\mathrm{gph}(A_m),\mathrm{gph}(A))\to0.
\]
Suppose not. Then there exist $\varepsilon>0$, a subsequence, still denoted by $m$,
and points $x_m\in H$ such that
\[
(x_m,A_m(x_m))\in K
\]
and
\[
\mathrm{dist}\bigl((x_m,A_m(x_m)),\mathrm{gph}(A)\bigr)\geq\varepsilon.
\]
Since $K$ is compact, up to a further subsequence,
\[
(x_m,A_m(x_m))\to(x,y)\in K.
\]
In particular, $x_m\in K_0\subset K_*$. Hence, by \eqref{eq:Am-Yosida-unif},
\[
\|A_m(x_m)-A_{\lambda_m}(x_m)\|_H\to0.
\]
Therefore
\[
(x_m,A_{\lambda_m}(x_m))\to(x,y).
\]
Since $A_{\lambda_m}\xrightarrow{g}A$ as $\lambda_m\to0$, the outer limit condition gives
\[
(x,y)\in\mathrm{gph}(A).
\]
This contradicts
\[
\mathrm{dist}\bigl((x_m,A_m(x_m)),\mathrm{gph}(A)\bigr)\geq\varepsilon.
\]
Hence
\[
h_K(\mathrm{gph}(A_m),\mathrm{gph}(A))\to0.
\]

\medskip

Next, we show
\[
h_K(\mathrm{gph}(A),\mathrm{gph}(A_m))\to 0.
\]
Let $(x,y)\in \mathrm{gph}(A)\cap K$.
Define
\[
z_m:=x+\lambda_m y.
\]
Since $y\in A x$, we have
\[
z_m=x+\lambda_m y\in x+\lambda_m A x,
\]
and hence
\[
J_{\lambda_m}z_m=x.
\]
Therefore
\[
A_{\lambda_m}(z_m)
=
\frac{z_m-J_{\lambda_m}z_m}{\lambda_m}
=
\frac{x+\lambda_m y-x}{\lambda_m}
=
y.
\]
Moreover, by definition of $K_1$, we have $z_m\in K_1\subset K_*$. Hence, using
\eqref{eq:Am-Yosida-unif},
\[
\|A_m(z_m)-y\|_H
=
\|A_m(z_m)-A_{\lambda_m}(z_m)\|_H
\to 0
\]
uniformly for $(x,y)\in\mathrm{gph}(A)\cap K$.
Also,
\[
\|z_m-x\|_H=\lambda_m\|y\|_H\to0
\]
uniformly for $(x,y)\in K$, since $K$ is compact and therefore bounded.
Thus
\[
\mathrm{dist}\bigl((x,y),\mathrm{gph}(A_m)\bigr)
\leq
\|(x,y)-(z_m,A_m(z_m))\|_{H\times H}
\to0
\]
uniformly for $(x,y)\in\mathrm{gph}(A)\cap K$.
Hence
\[
h_K(\mathrm{gph}(A),\mathrm{gph}(A_m))\to0.
\]

Combining the two estimates gives
\[
d_{\mathrm{gph},K}(A_m,A)\to0.
\]
This completes the proof.
\end{proof}

\subsection{Proof of Corollary~\ref{cor-main-2}}
\label{app:cor-main-2}

\begin{proof}
Let
\[
J_\lambda:=(I+\lambda A)^{-1}
\]
be the resolvent of $A$, and define
\[
R_\lambda:=2J_\lambda-I.
\]
Since $A$ is maximally monotone, $J_\lambda$ is firmly nonexpansive. Hence
$R_\lambda$ is nonexpansive, namely,
\[
\|R_\lambda(x)-R_\lambda(y)\|_H
\le
\|x-y\|_H
\qquad
(x,y\in H).
\]
Moreover,
\[
A_\lambda
:=
\frac{I-J_\lambda}{\lambda}
=
\frac{I-R_\lambda}{2\lambda}.
\]
Fix a compact set $K\subset H\times H$.
Set
\[
K_0:=\pi_1(K),
\]
where $\pi_1:H\times H\to H$ denotes the first projection.

Define
\[
K_1
:=
\overline{
\{
x+\lambda_m y
:
(x,y)\in \operatorname{gph}(A)\cap K,\ m\in\mathbb N
\}
}.
\]
Since $C$ is compact and $\lambda_m\to0$, the set $K_1$ is compact.
Put
\[
K_*:=K_0\cup K_1.
\]
Then $K_*$ is compact.

Since each $R_{\lambda_m}:H\to H$ is $1$-Lipschitz,
the EDAP and the universality of $\mathcal F^{\mathrm{Lip}}$
yield a sequence
\[
(\psi_m)_{m\in\mathbb N}\subset\mathcal F^{\mathrm{Lip}}
\]
such that by the diagonal argument
\[
\sup_{u\in K_*}
\|
D_m\circ\psi_m\circ E_m(u)
-
R_{\lambda_m}(u)
\|_H
=
o(\lambda_m).
\]
Set
\[
S_m:=D_m\circ\psi_m\circ E_m,
\qquad
B_m:=\frac{I-S_m}{2\lambda_m}.
\]
Then
\[
S_m:H\to H
\]
is $1$-Lipschitz.
We see that 
\[
\begin{aligned}
\sup_{u\in K_*}
\|B_m(u)-A_{\lambda_m}(u)\|_H
&=
\sup_{u\in K_*}
\left\|
\frac{I-S_m}{2\lambda_m}(u)
-
\frac{I-R_{\lambda_m}}{2\lambda_m}(u)
\right\|_H
\\
&\le
\frac1{2\lambda_m}
\sup_{u\in K_*}
\|S_m(u)-R_{\lambda_m}(u)\|_H
\\
&\to0.
\end{aligned}
\]
Hence
\begin{equation}
\label{eq:Bm-Yosida}
\sup_{u\in K_*}
\|B_m(u)-A_{\lambda_m}(u)\|_H
\to0.
\end{equation}

We prove that
\[
d_{\mathrm{gph},K}(B_m,A)\to0.
\]

First, we show
\[
h_K(\operatorname{gph}(B_m),\operatorname{gph}(A))
\to0.
\]
Suppose not.
Then there exist $\varepsilon>0$, a subsequence still denoted by $m$,
and points $x_m\in H$ such that
\[
(x_m,B_m(x_m))\in K
\]
and
\[
\operatorname{dist}
\bigl(
(x_m,B_m(x_m)),
\operatorname{gph}(A)
\bigr)
\ge\varepsilon.
\]
Since $K$ is compact, after passing to a further subsequence,
\[
(x_m,B_m(x_m))\to(x,y)\in K.
\]
Since $x_m\in K_0\subset K_*$,
\eqref{eq:Bm-Yosida} implies
\[
\|B_m(x_m)-A_{\lambda_m}(x_m)\|_H\to0.
\]
Hence
\[
(x_m,A_{\lambda_m}(x_m))\to(x,y).
\]

Since $A_{\lambda_m}\xrightarrow{g}A$ as $\lambda_m\to0$,
the outer limit condition yields
\[
(x,y)\in\operatorname{gph}(A),
\]
which contradicts the assumption.
Therefore
\[
h_K(\operatorname{gph}(B_m),\operatorname{gph}(A))
\to0.
\]

Next, we show
\[
h_K(\operatorname{gph}(A),\operatorname{gph}(B_m))
\to0.
\]
Let
\[
(x,y)\in\operatorname{gph}(A)\cap K.
\]
Define
\[
z_m:=x+\lambda_m y.
\]
Since $y\in Ax$,
\[
z_m\in x+\lambda_m Ax,
\]
and therefore
\[
J_{\lambda_m}z_m=x.
\]
Hence
\[
R_{\lambda_m}(z_m)
=
2x-z_m
=
x-\lambda_m y.
\]
Consequently,
\[
A_{\lambda_m}(z_m)
=
\frac{z_m-J_{\lambda_m}z_m}{\lambda_m}
=
y.
\]

Since $z_m\in K_1\subset K_*$,
\eqref{eq:Bm-Yosida} gives
\[
\|B_m(z_m)-y\|_H
=
\|B_m(z_m)-A_{\lambda_m}(z_m)\|_H
\to0.
\]
Moreover,
\[
\|z_m-x\|_H
=
\lambda_m\|y\|_H
\to0
\]
uniformly for $(x,y)\in K$,
since $K$ is bounded.

Therefore
\[
\operatorname{dist}
\bigl(
(x,y),
\operatorname{gph}(B_m)
\bigr)
\le
\|(x,y)-(z_m,B_m(z_m))\|_{H\times H}
\to0
\]
uniformly for
\[
(x,y)\in\operatorname{gph}(A)\cap K.
\]
Hence
\[
h_K(\operatorname{gph}(A),\operatorname{gph}(B_m))
\to0.
\]

Combining the above two estimates,
\[
d_{\mathrm{gph},K}(B_m,A)\to0.
\]

Finally, we show that $B_m$ is maximally monotone.
Since $S_m$ is $1$-Lipschitz,
for any $x,y\in H$,
\[
\begin{aligned}
\langle B_mx-B_my,x-y\rangle
&=
\frac1{2\lambda_m}
\Bigl(
\|x-y\|_H^2
-
\langle S_mx-S_my,x-y\rangle
\Bigr)
\\
&\ge
\frac1{2\lambda_m}
\Bigl(
\|x-y\|_H^2
-
\|S_mx-S_my\|_H\|x-y\|_H
\Bigr)
\\
&\ge0.
\end{aligned}
\]
Thus $B_m$ is monotone.
Since $B_m:H\to H$ is continuous and everywhere defined,
it is maximally monotone.
\end{proof}

\section{Examples of maximally monotone operators}
\label{app:ex-max-mono}

We give some examples of maximally monotone operators. 

\begin{enumerate}[(a)]
    \item Subdifferentials of convex functionals $H \to \RI$ such as 
    \begin{itemize}
        \item the $p$-Laplacian, $1 < p < +\infty$,  $A : W^{1,p}(\Omega) \subset L^2(\Omega) \to L^2(\Omega)$ for a bounded open set $\Omega \subset \mathbb{R}^d$, defined as $A(u) = -\mathrm{div}(|\nabla u|^{p-2} \nabla u)$, which is a subgradient of the $p$-Dirichlet energy functional $J \colon L^2(\Omega) \to \RI$ given by
    \begin{equation*}
        J(u) =
        \begin{cases}
        \frac1p\int_\Omega \norm{\grad u}^p \, dx,  & u \in W^{1,p}(\Omega), \\
        +\infty,  & u \in L^2(\Omega) \setminus W^{1,p}(\Omega).
        \end{cases}
    \end{equation*}
    \item the $1$-Laplacian $A: \bv(\Omega)(\Omega) \subset L^2(\Omega) \to L^2(\Omega)$, where $\bv(\Omega)$ is the space of functions of bounded variation, formally written as $A(u) = -\div\left(\frac{\grad u}{\abs{\grad u}}\right)$ but rigorously defined as the subgradient of minimal norm of the Total Variation functional $J \colon L^2(\Omega) \to \RI$
    \begin{equation*}
        J(u) = 
        \begin{cases}
        \sup\left\lbrace \int_\Omega u\div \phi \d x \st \phi\in C^\infty_c(\Omega;\R^n),\;\sup_{\Omega}\abs{\phi}\leq 1 \right\rbrace,  & u \in \bv(\Omega), \\
        +\infty,  & u \in L^2(\Omega) \setminus \bv(\Omega).
        \end{cases}
    \end{equation*}
    \item the $\infty$-Laplacian~\cite{aronsson2004tour, juutinen1999infinity}, which can be formally written as $\Delta_\infty u:=\langle\grad u, D^2 u\grad u\rangle$ for sufficiently smooth functions but is rigorously defined in the viscosity sense; it is also a subgradient of the Lipschitz constant functional defined on the Banach space $C(\Omega)$
    \begin{equation*}
        J(u) = 
        \begin{cases}
        \norm{\grad u}_{\L^\infty},  & u \in W^{1,\infty}(\Omega), \\
        +\infty,  & u \in C(\Omega) \setminus W^{1,\infty}(\Omega)
        \end{cases}
    \end{equation*}
    and, similarly to the $p$-Laplacian, has a divergence form $A(u) = -\div(\tau\grad_{{\tau}}u)$ where $\tau$ is a certain measure that depends on $u$ and $\grad_{{\tau}}$ is the tangential gradient~\cite{bungert2022eigenvalue}. 
    \item the operator $u \mapsto \Delta (u^{[m]})$ where $m>0$ is a constant and $u^{[m]} := \sign(u)\abs{u}^{m}$, which arises in the porous medium ($m>1$) and the fast diffusion ($m<1$) equations; it is a subgradient of the following functional defined on the negative-order Sobolev space $H^{-1}(\Omega)$ (the dual of the Sobolev space $H^1_0(\Omega)$ with homogeneous Dirichlet boundary conditions)
    \begin{equation*}
    J(u) := 
    \begin{cases}
    \frac{1}{m+1}\int_\Omega\abs{u}^{m+1} \, d x, \quad&\text{if }u\in L^{m+1}(\Omega) \cap H^{-1}(\Omega),\\
    \infty,\quad & u \in H^{-1}(\Omega) \setminus L^{m+1}(\Omega).
    \end{cases}
    \end{equation*}
    \end{itemize}
    More details can be found, for example, in~\cite{bungert2020asymptotic, bungert2022gradient, bungert2025introduction}.
    \item The differentiation operator (the inverse Volterra operator~\cite{bauschke2010examples}) on periodic functions $A : H^1_{\mathrm{per}}(0,1) \subset L^2_{\mathrm{per}}(0,1) \to L^2_{\mathrm{per}}(0,1)$, $A(u)(x) = \frac{d}{dx}u(x)$ (which is not symmetric and therefore cannot be the subdifferential of a convex function). 
  %
  %
  \item The inverse $A^{-1}$ of a maximally monotone operator $A$ is maximally monotone; see \cite[Prop. 20.22]{bauschke2011convex}. Thus, the inverses of the above operators are also maximally monotone.
\end{enumerate}

\section{Quantitative result}
\label{app:cor:main-quant}

Using quantitative approximation results for neural networks applied to the Yosida approximation, we obtain the following.

\begin{corollary}[Quantitative local graph estimate]
\label{cor:main-quant}
Assume that $H$ has the EDAP with encoder and decoder mappings
$E_m:H\to \mathbb R^{w_m}$ and $D_m:\mathbb R^{w_m}\to H$.
Let $A:D(A)\subset H\to 2^H$ be maximally monotone and assume that
$0\in D(A)$. Let $K=K_1\times K_2\subset H\times H$ be compact with
$0\in K_1$. For $\lambda\in(0,1)$ set
\[
S_\lambda(K):=
K_1\cup \{x+\lambda y:(x,y)\in \operatorname{gph}(A)\cap K\}.
\]
Then, for every $\lambda\in(0,1)$ and $m\in\mathbb N$, there exists
a ReLU network $\varphi_{\lambda,m,N}:\mathbb R^{w_m}\to\mathbb R^{w_m}$
with size parameter $N\in\mathbb N$ such that, with
\[
B_{\lambda,m,N}:=
D_m\circ \varphi_{\lambda,m,N}\circ E_m,
\]
one has
\[
\begin{aligned}
d_{\mathrm{gph},K}(B_{\lambda,m,N},A)
&\le
\lambda R_K
+(1+\lambda)
\Bigg[
\frac{1}{\lambda}C_{S_\lambda(K)}(w_m)
+
C_{A_\lambda\circ D_m\circ E_m(S_\lambda(K))}(w_m)
\\
&\quad
+
C_{\mathrm{Yar}}(w_m)
\left(
\|A^0(0)\|_H
+
\frac{1}{\lambda}
\bigl(1+\operatorname{diam}(S_\lambda(K))\bigr)
\right)
\operatorname{diam}(S_\lambda(K))N^{-1/w_m}
\Bigg],
\end{aligned}
\]
where
\[
R_K:=\sup_{y\in K_2}\|y\|_H,
\qquad
C_B(w_m):=\sup_{u\in B}\|u-D_m\circ E_m(u)\|_H.
\]
\end{corollary}


Corollary~\ref{cor:main-quant} highlights the trade-off between $\lambda$, $m$, and $N$. 
The first term  vanishes as $\lambda \to 0$, whereas the second and third terms and the fourth term may increase. 
On the other hand, for fixed $\lambda > 0$, the second and third terms vanish as $m \to \infty$. 
Moreover, for fixed $\lambda > 0$ and $m \in \mathbb{N}$, the fourth term vanishes as $N \to \infty$.

\begin{remark}
\label{rem:quantitative-cor}
We give several remarks on the upper bound estimates obtained in the above. 
\begin{enumerate}[(a)]
    \item Although we obtain an upper bound on the local graph distance between the target operator and the encoder--decoder architecture, some terms are not fully explicit. While the constant $C_{\mathrm{Yar}}(w_m)$ arising from the Yarotsky-type approximation theorem is generally not explicit, 
    the encoder--decoder errors $C_B(w_m)$ may admit explicit bounds under additional structural assumptions on 
    the encoder and decoder mappings $E_m, D_m$. 
    \item We emphasize that although $N$ does not depend on the compact, the distance $d_{\mathrm{gph},K}$, in which the rate is obtained, does depend on $K$ (see Definition~\ref{def:graph-distance}). Therefore, despite its appearance, this error estimate is not uniform over compact sets.
\end{enumerate}

\end{remark}

\begin{proof}
Set
\[
B_{\lambda,m,N}:=
D_m\circ \varphi_{\lambda,m,N}\circ E_m.
\]
Define
\[
\eta_{\lambda,m,N}
:=
\sup_{u\in S_\lambda(K)}
\|B_{\lambda,m,N}(u)-A_\lambda(u)\|_H .
\]
We claim that
\[
d_{\mathrm{gph},K}(B_{\lambda,m,N},A)
\le
\lambda R_K+(1+\lambda)\eta_{\lambda,m,N}.
\]

Indeed, let $(u,B_{\lambda,m,N}(u))\in \operatorname{gph}(B_{\lambda,m,N})\cap K$.
Set
\[
q:=A_\lambda(u),
\qquad
p:=J_\lambda u.
\]
Then $q\in A(p)$ and $u=p+\lambda q$. Hence $(p,q)\in \operatorname{gph}(A)$.
Moreover, since $B_{\lambda,m,N}(u)\in K_2$,
\[
\|q\|_H
\le
\|B_{\lambda,m,N}(u)\|_H+\|q-B_{\lambda,m,N}(u)\|_H
\le
R_K+\eta_{\lambda,m,N}.
\]
Therefore
\[
\begin{aligned}
\operatorname{dist}\bigl((u,B_{\lambda,m,N}(u)),\operatorname{gph}(A)\bigr)
&\le
\|(u,B_{\lambda,m,N}(u))-(p,q)\|_{H\times H}  \\
&\le
\lambda\|q\|_H+\|B_{\lambda,m,N}(u)-q\|_H  \\
&\le
\lambda R_K+(1+\lambda)\eta_{\lambda,m,N}.
\end{aligned}
\]
Taking the supremum over
$\operatorname{gph}(B_{\lambda,m,N})\cap K$ gives
\[
h_K(\operatorname{gph}(B_{\lambda,m,N}),\operatorname{gph}(A))
\le
\lambda R_K+(1+\lambda)\eta_{\lambda,m,N}.
\]

Conversely, let $(x,y)\in \operatorname{gph}(A)\cap K$ and set
\[
z:=x+\lambda y.
\]
Then $z\in S_\lambda(K)$ and, by the resolvent identity,
\[
J_\lambda z=x,
\qquad
A_\lambda(z)=y.
\]
Hence
\[
(z,B_{\lambda,m,N}(z))\in \operatorname{gph}(B_{\lambda,m,N}),
\]
and therefore
\[
\begin{aligned}
\operatorname{dist}\bigl((x,y),\operatorname{gph}(B_{\lambda,m,N})\bigr)
&\le
\|(x,y)-(z,B_{\lambda,m,N}(z))\|_{H\times H}  \\
&\le
\lambda\|y\|_H+\|y-B_{\lambda,m,N}(z)\|_H  \\
&\le
\lambda R_K+\eta_{\lambda,m,N}.
\end{aligned}
\]
Taking the supremum over $\operatorname{gph}(A)\cap K$ yields
\[
h_K(\operatorname{gph}(A),\operatorname{gph}(B_{\lambda,m,N}))
\le
\lambda R_K+\eta_{\lambda,m,N}.
\]
Combining the two one-sided estimates gives
\[
d_{\mathrm{gph},K}(B_{\lambda,m,N},A)
\le
\lambda R_K+(1+\lambda)\eta_{\lambda,m,N}.
\]

It remains to bound $\eta_{\lambda,m,N}$. For $u\in S_\lambda(K)$, we decompose
\[
\begin{aligned}
\|A_\lambda(u)-B_{\lambda,m,N}(u)\|_H
&\le
\|A_\lambda(u)-A_\lambda(D_mE_m u)\|_H  \\
&\quad
+\|A_\lambda(D_mE_m u)-D_mE_mA_\lambda(D_mE_m u)\|_H \\
&\quad
+\|D_mE_mA_\lambda(D_mE_m u)
      -D_m\varphi_{\lambda,m,N}E_m u\|_H .
\end{aligned}
\]
Since $A_\lambda$ is $1/\lambda$-Lipschitz, the first term is bounded by
\[
\frac{1}{\lambda}C_{S_\lambda(K)}(w_m).
\]
The second term is bounded by
\[
C_{A_\lambda\circ D_m\circ E_m(S_\lambda(K))}(w_m).
\]

For the third term, note that $E_m \circ A_\lambda \circ D_m$ is Lipschitz on $E_m(S_\lambda(K))$, and hence belongs to 
$W^{1,\infty}(E_m(S_\lambda(K)); \mathbb{R}^{w_m})$. 
By the Yarotsky-type ReLU approximation theorem \cite[Thm. 1]{yarotsky2017error}, 
there exists a ReLU network $\varphi_{\lambda,m,N} : \mathbb{R}^{w_m} \to \mathbb{R}^{w_m}$ such that
\[
\begin{aligned}
\sup_{z \in E_m(S_\lambda(K))}
\|E_m \circ A_\lambda \circ D_m(z) - \varphi_{\lambda,m,N}(z)\|
&\le
C_{\mathrm{Yar}}(w_m)
\|E_m \circ A_\lambda \circ D_m\|_{W^{1,\infty}(E_m(S_\lambda(K)); \mathbb{R}^{w_m})} \\
&\quad \times \mathrm{diam}(E_m(S_\lambda(K))) \, N^{-1/w_m}.
\end{aligned}
\]

Since $\mathrm{diam}(E_m(S_\lambda(K))) \le \mathrm{diam}(S_\lambda(K))$, we estimate (by $0 \in K_1$)
\[
\begin{aligned}
\|E_m \circ A_\lambda \circ D_m\|_{W^{1,\infty}}
&\le
\|E_m \circ A_\lambda \circ D_m\|_{L^\infty}
+ \mathrm{Lip}(E_m \circ A_\lambda \circ D_m) \\
&\le
\|A^0(0)\|_H + \frac{1}{\lambda} \mathrm{diam}(S_\lambda(K)) + \frac{1}{\lambda}.
\end{aligned}
\]

Thus,
\[
\begin{aligned}
\sup_{z \in E_m(S_\lambda(K))}
&\|E_m \circ A_\lambda \circ D_m(z) - \varphi_{\lambda,m,N}(z)\|
\\
&
\le
C_{\mathrm{Yar}}(w_m)
\left(\|A^0(0)\|_H + \frac{1}{\lambda}(1 + \mathrm{diam}(S_\lambda(K)))\right)
\mathrm{diam}(S_\lambda(K_1))\, N^{-1/w_m}.
\end{aligned}
\]

Since $D_m$ is $1$-Lipschitz, the same bound applies after composing with $D_m$.
Substituting these estimates into the deterministic local graph estimate proves the result.
\end{proof}

\section{Details of Experiments}
\label{app:details-experiment}

We consider the following two target operators:
\begin{align*}
& A_1 : u(t) \mapsto u'(t),
\\
&
A_2 : u(x,y) \mapsto 
-\operatorname{div}\bigl(|\nabla u(x,y)|^{p-2}\nabla u(x,y)\bigr),
\end{align*}
defined on $[0,1]$ and $[0,1]^2$, respectively. We generate input data as
\begin{align}
\label{eq:ex-generate-1}
& u(t)
=
\sum_{j=1}^{K}
\frac{1}{n_j^{\beta}}
\Bigl(
a_j \sin(2\pi n_j t)
+
b_j \cos(2\pi n_j t)
\Bigr),
\\
\label{eq:ex-generate-2}
& u(x,y)
=
\sum_{j=1}^{K}
\frac{1}{(\sqrt{k_j^2+\ell_j^2})^\beta}
\Bigl[
a_j \sin(2\pi k_j x)\sin(2\pi \ell_j y)
+
b_j \sin(2\pi k_j x)\cos(2\pi \ell_j y)
\\
&
\qquad\qquad
+
c_j \cos(2\pi k_j x)\sin(2\pi \ell_j y)
+
d_j \cos(2\pi k_j x)\cos(2\pi \ell_j y)
\Bigr].
\nonumber
\end{align}
Here,
\[
n_j \sim \mathrm{Uniform}\{n_{\min}^{\mathrm{train}},\dots,n_{\max}^{\mathrm{train}}\}
\]
for the training data, and
\[
n_j \sim \mathrm{Uniform}\{n_{\min}^{\mathrm{test}},\dots,n_{\max}^{\mathrm{test}}\}
\]
for the test data. Moreover,
\[
a_j,b_j,c_j,d_j \sim \mathcal{N}(0,1)
\]
are independent random variables. The training and test datasets are constructed as
\[
\{(u_j,A_k(u_j))\}_{j=1}^{N^{\mathrm{train}}}
\quad\text{and}\quad
\{(u_j,A_k(u_j))\}_{j=1}^{N^{\mathrm{test}}},
\]
respectively, where each $u_j$ is sampled according to \eqref{eq:ex-generate-1} or \eqref{eq:ex-generate-2}, and where
$N^{\mathrm{train}}$ and $N^{\mathrm{test}}$ denote the numbers of training and test samples.

This data generation is inspired by Examples~\ref{ex:uniform-conv}
and~\ref{ex:Lp-conv}, which show that differential operators cannot,
in general, be approximated under uniform or \(L^p\) convergence,
thereby motivating the use of graph-distance-based losses.
In particular, high-frequency inputs of the form
\eqref{eq:ex-generate-1} and \eqref{eq:ex-generate-2} are mapped to
outputs with significantly amplified amplitudes, reflecting the
unbounded nature of differential operators (see
Figures~\ref{fig:high-frequency-derivative}
and~\ref{fig:high-frequency-plaplacian}).

\begin{figure}[h]
\centering

\begin{minipage}[t]{0.49\linewidth}
    \centering
    \includegraphics[width=\linewidth]{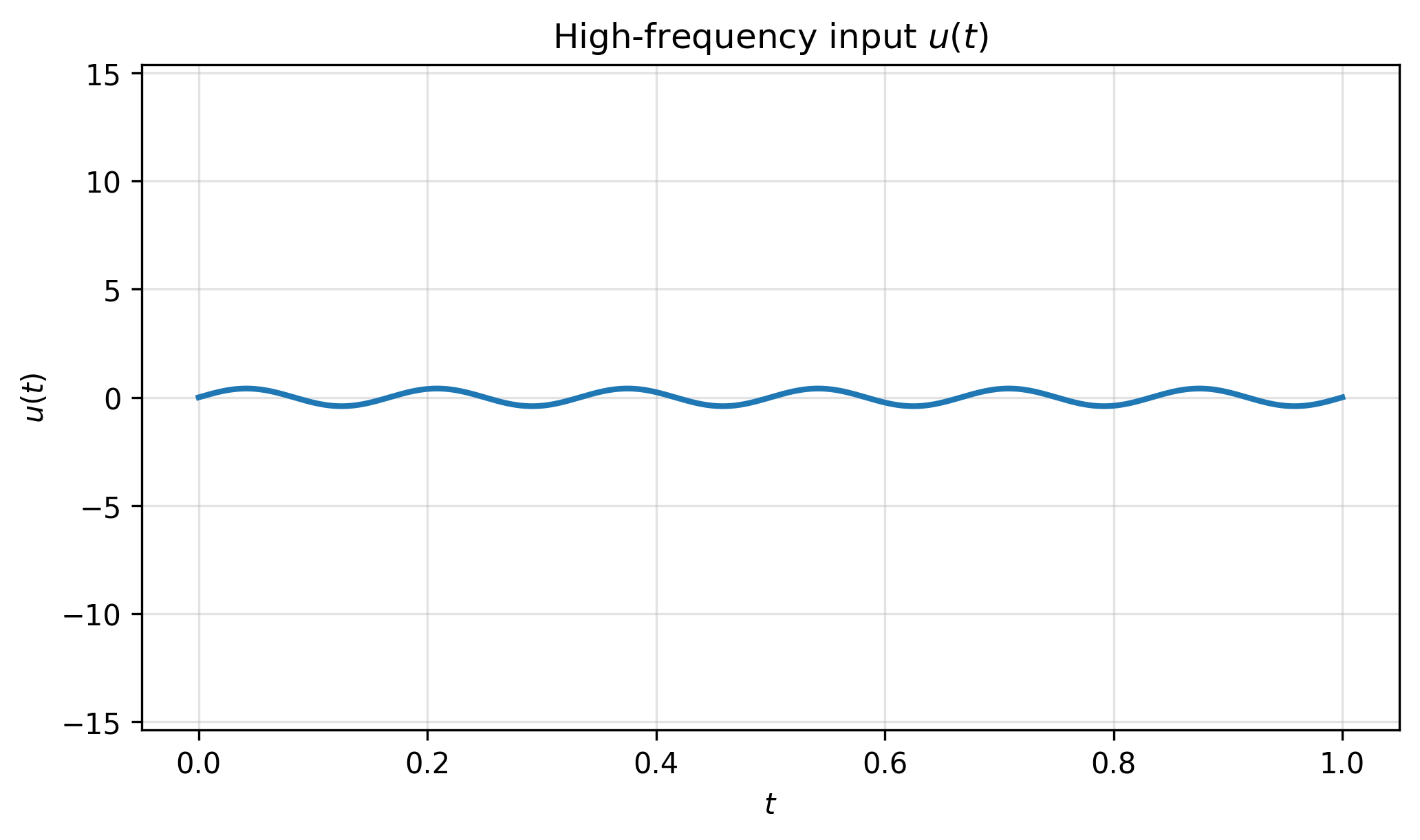}
\end{minipage}
\hfill
\begin{minipage}[t]{0.49\linewidth}
    \centering
    \includegraphics[width=\linewidth]{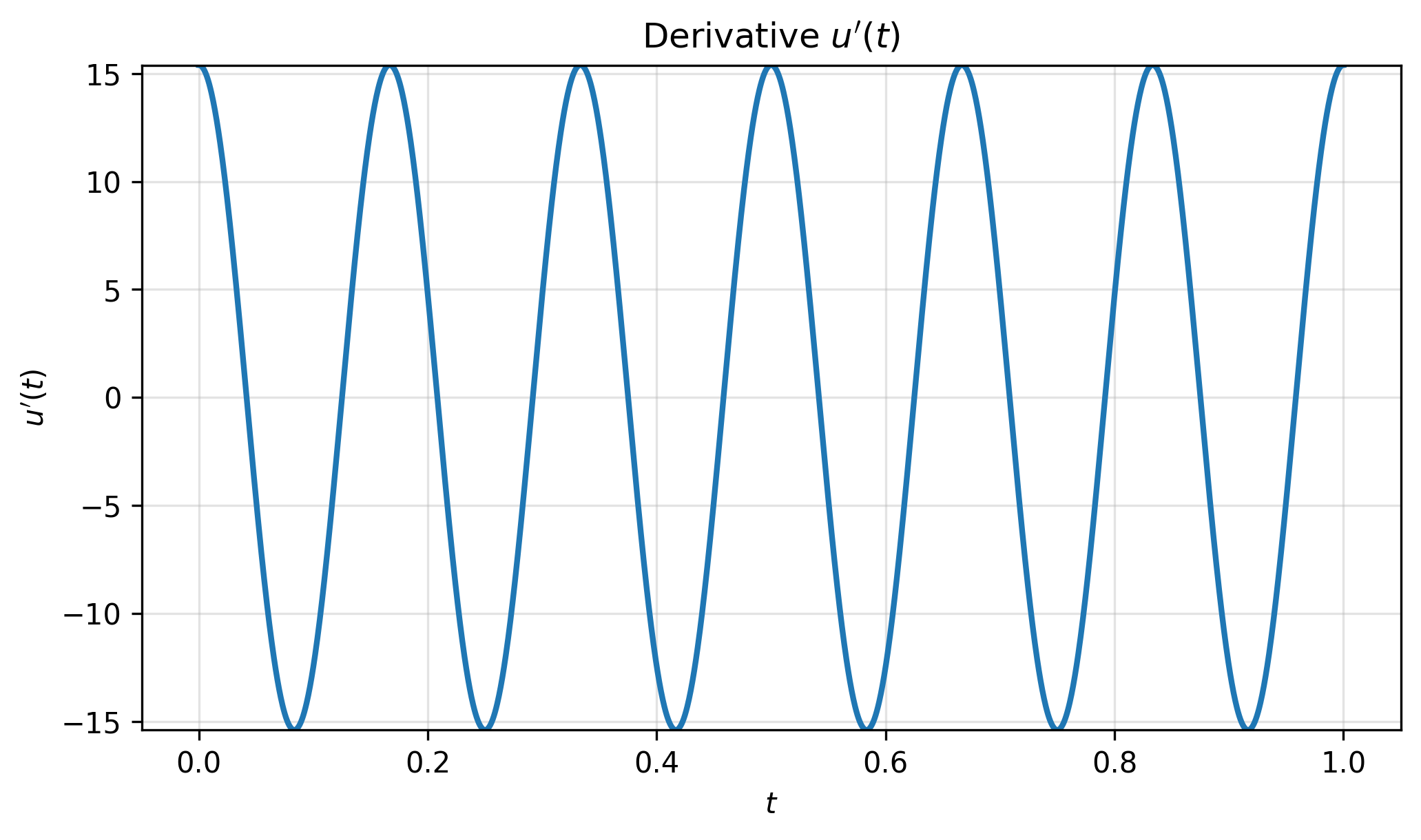}
\end{minipage}

\caption{
High-frequency input \(u(t)\) (left) and its derivative \(u'(t)\) (right). Here $u(t)$ is generated as in \eqref{eq:ex-generate-1} with $K=1$, $n=6$, $a_j=1$, $b_j=0$, and $\beta=0.5$. 
}
\label{fig:high-frequency-derivative}
\end{figure}

\begin{figure}[h]
\centering

\begin{minipage}[t]{0.49\linewidth}
    \centering
    \includegraphics[width=\linewidth]{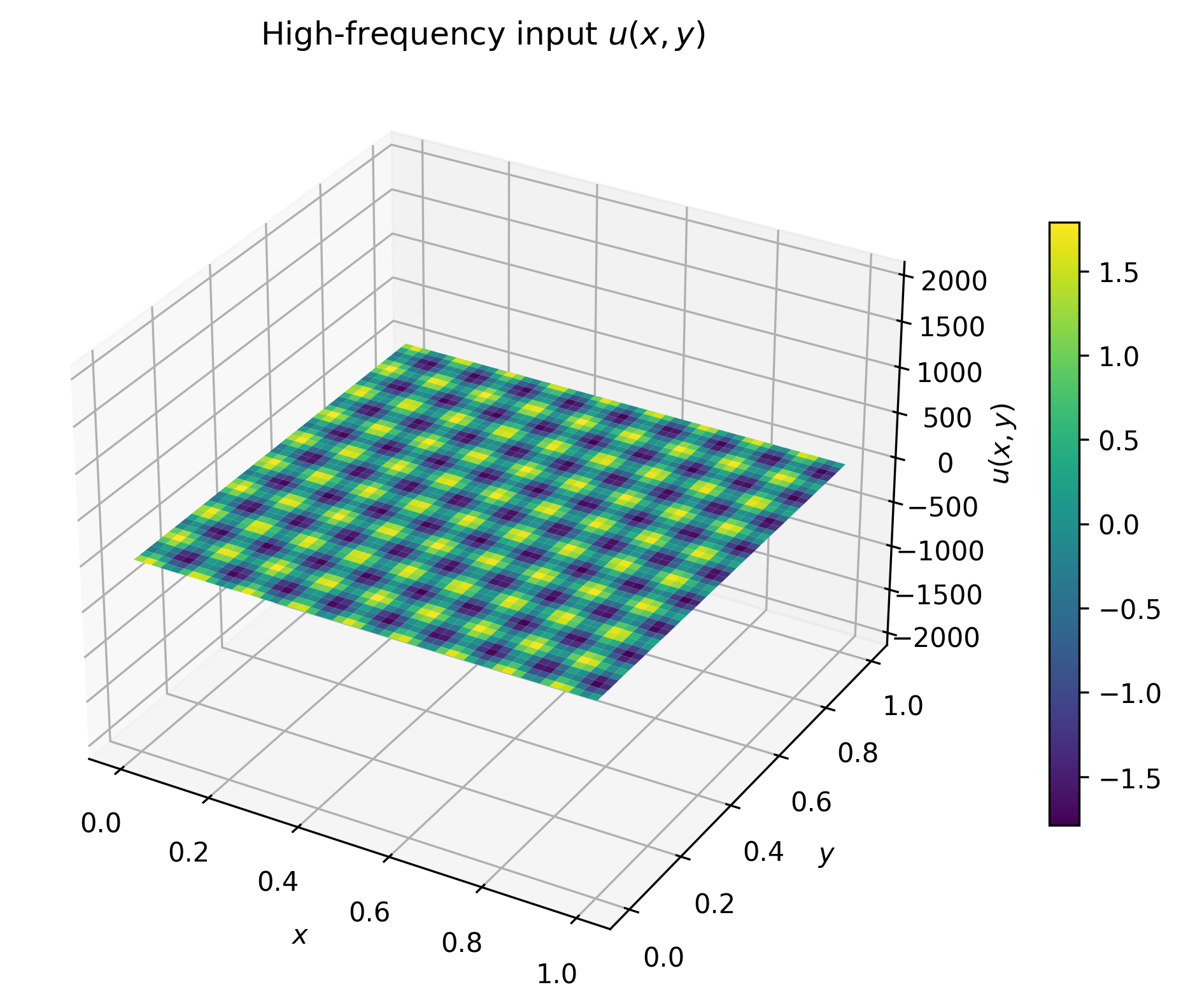}
\end{minipage}
\hfill
\begin{minipage}[t]{0.49\linewidth}
    \centering
    \includegraphics[width=\linewidth]{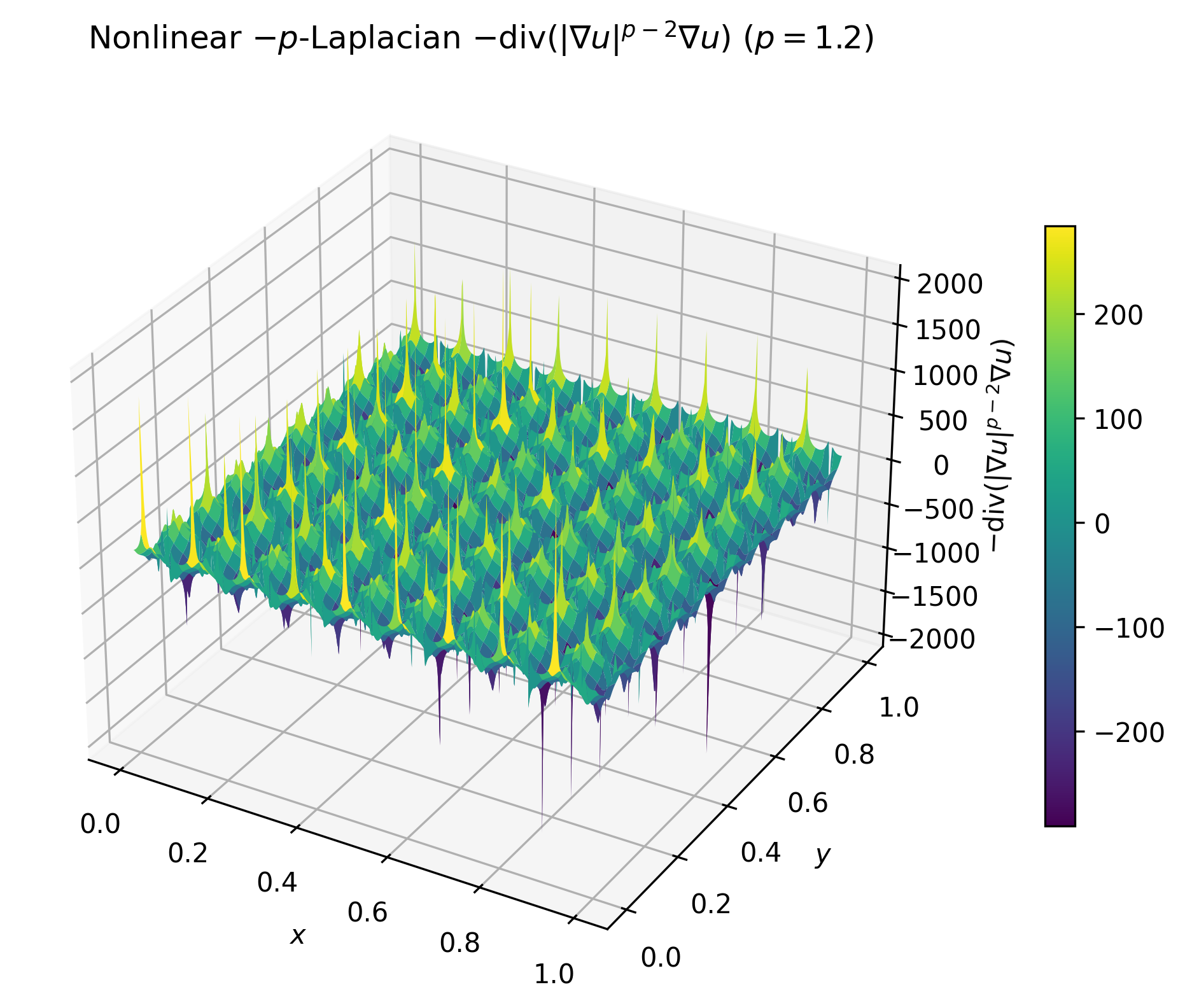}
\end{minipage}

\caption{
High-frequency input \(u(x,y)\) (left) and the corresponding nonlinear \(p\)-Laplacian
\(
-\mathrm{div}(|\nabla u|^{p-2}\nabla u)
\)
with \(p=1.2\) (right). Here $u(x,y)$ is generated as in \eqref{eq:ex-generate-2} with $K=1$, $n=9$, $a_j=1$, $b_j=1$, and $\beta=0$. 
}
\label{fig:high-frequency-plaplacian}
\end{figure}


We compare several models $\hat{A}$ based on Fourier Neural Operators (FNOs), trained with different loss functions:
\begin{equation}
\label{eq:ell2-loss}
\sum_{j=1}^{N^{\rm train}}
\| A_k(u_j) - \hat{A}(u_j) \|_{L^2}^2
\qquad
\text{($\ell^2$ loss)},
\end{equation}
\begin{equation}
\label{eq:soft-ell-loss}
\max_{j \in [N^{\rm train}]}
\| A_k(u_j) - \hat{A}(u_j) \|_{L^2}
\;\approx\;
\tau_{\infty}
\log\!\left(
\sum_{j=1}^{N^{\rm train}}
\exp\!\left(
\frac{
\| A_k(u_j) - \hat{A}(u_j) \|_{L^2}
}{\tau_{\infty}}
\right)
\right)
\qquad
\text{(soft $\ell^\infty$ loss)},
\end{equation}
and a graph-distance-based loss defined via the soft approximation
\begin{align}
\max\Bigl\{
&
\sup_{i \in [N^{\rm train}]}
\inf_{j \in [N^{\rm train}]}
d_{ij},
\;
\sup_{j \in [N^{\rm train}]}
\inf_{i \in [N^{\rm train}]}
d_{ij}
\Bigr\}
\nonumber
\\
&
\qquad\approx\;
\max\{H,H'\}
=: d_{\mathrm{soft\text{-}gph}}
(A_k,\hat{A})
\qquad
\text{(soft graph loss)},
\label{eq:soft-graph-loss}
\end{align}
with weighted distance
\[
d_{ij}
=
\sqrt{
w_1 \|u_i-u_j\|_{L^2}^2
+
w_2 \|A_k(u_i)-\hat{A}(u_j)\|_{L^2}^2
}.
\]
The soft approximations are defined by
\[
H
=
\tau_{\rm gph,out}
\log
\sum_{i=1}^{N^{\rm train}}
\exp\!\left(
\frac{
-\tau_{\rm gph,in}
\log
\sum_{j=1}^{N^{\rm train}}
\exp(-d_{ij}/\tau_{\rm gph,in})
}{
\tau_{\rm gph,out}
}
\right),
\]
and
\[
H'
=
\tau_{\rm gph,out}
\log
\sum_{j=1}^{N^{\rm train}}
\exp\!\left(
\frac{
-\tau_{\rm gph,in}
\log
\sum_{i=1}^{N^{\rm train}}
\exp(-d_{ij}/\tau_{\rm gph,in})
}{
\tau_{\rm gph,out}
}
\right).
\]

For a fair comparison, all models use the same FNO architecture,
including the same number of modes, width, and depth.
All models are trained using Adam.

\medskip
In addition, we introduce a structured model inspired by the
resolvent formulation of maximally monotone operators
(Corollary~\ref{cor-main-2}).
Instead of directly parameterizing $\hat{A}$, we define
\begin{equation}
\label{eq:structured-graph}
\hat{S}(u)
=
u-\lambda \hat{R}(u),
\qquad
\hat{A}(u)
=
\frac{u-\hat{S}(u)}{\lambda},
\end{equation}
where $\hat{R}$ is represented by an FNO.
Note that this parametrization does not restrict the expressive power
of $\hat{A}$, since formally $\hat{A}(u)=\hat{R}(u)$.
However, it allows us to impose structural constraints on $\hat{S}$,
which are crucial for preserving maximal monotonicity.
To promote this structure, we add a penalty encouraging $\hat{S}$ to
be non-expansive:
\[
\begin{aligned}
\mathcal{L}_{\mathrm{nonexp}}
=
&
\sum_{(p,q)}
\left[
\max\left\{
0,\,
\frac{
\|\hat{S}(u_p)-\hat{S}(u_q)\|_{L^2}
}{
\|u_p-u_q\|_{L^2}
}
-1
\right\}^2
\right],
\\
&
\qquad\qquad
\text{where } \ 
u_p,u_q \sim (u_j)_{j=1}^{N^{\rm train}}
\quad
\text{with}
\quad
\#\{(u_p,u_q)\}=N^{\rm nonexp}.
\end{aligned}
\]
This term penalizes violations of the non-expansive condition
\[
\|\hat{S}(x)-\hat{S}(y)\|_{L^2}
\le
\|x-y\|_{L^2}.
\]

The structured model is trained using the loss
\[
\mathcal{L}
=
d_{\mathrm{soft\text{-}gph}}
(A,\hat{A})
+
\gamma \mathcal{L}_{\mathrm{nonexp}}.
\]

\medskip

In total, we train the following four models:
\begin{itemize}
\item FNO trained with $\ell^2$ loss,
\item FNO trained with soft $\ell^\infty$ loss,
\item FNO trained with soft graph loss,
\item Structured FNO trained with soft graph loss and non-expansive regularization.
\end{itemize}

We evaluate the trained models using the following metrics:
\begin{itemize}

    \item \textbf{Test MSE}
    \begin{equation}
    \label{eq:test-mse}
    \mathrm{Test\ MSE}
    :=
    \frac{1}{N^{\rm test}}
    \sum_{i=1}^{N^{\rm test}}
    \|\hat A(u_i)-A_k(u_i)\|_{L^2}^2.
    \end{equation}

    \item \textbf{Mean Rel-L2}
    \begin{equation}
    \label{eq:mean-rel-l2}
    \mathrm{Mean\ Rel\text{-}L2}
    :=
    \frac{1}{N^{\rm test}}
    \sum_{i=1}^{N^{\rm test}}
    \frac{
        \|\hat A(u_i)-A_k(u_i)\|_{L^2}
    }{
        \|A_k(u_i)\|_{L^2}
    }.
    \end{equation}

    \item \textbf{Worst Rel-L2}
    \begin{equation}
    \label{eq:worst-rel-l2}
    \mathrm{Worst\ Rel\text{-}L2}
    :=
    \max_{1\le i\le N^{\rm test}}
    \frac{
        \|\hat A(u_i)-A_k(u_i)\|_{L^2}
    }{
        \|A_k(u_i)\|_{L^2}
    }.
    \end{equation}

    \item \textbf{Test Graph}
    \begin{equation}
    \label{eq:test-graph}
    \mathrm{Test\ Graph}
    :=
    \max\Bigl\{
    \sup_{i \in [N^{\rm test}]}
    \inf_{j \in [N^{\rm test}]}
    d_{ij},\;
    \sup_{j \in [N^{\rm test}]}
    \inf_{i \in [N^{\rm test}]}
    d_{ij}
    \Bigr\}
    \end{equation}

    \item \textbf{Mono mean-viol}

    For randomly sampled pairs $(u_p,u_q)$, define
    \begin{equation}
    \label{eq:mono-inner}
    m_{pq}
    :=
    \bigl\langle
    u_p-u_q,
    \hat A(u_p)-\hat A(u_q)
    \bigr\rangle_H,
    \end{equation}
    where
    \[
    u_p,u_q \sim (u_j)_{j=1}^{N^{\rm train}},
    \qquad
    \#\{(u_p,u_q)\}=N^{\rm mono}.
    \]

    The mean monotonicity violation is defined by
    \begin{equation}
    \label{eq:mono-mean}
    \mathrm{Mono\ mean\text{-}viol}
    :=
    \frac{1}{N^{\rm mono}}
    \sum_{(p,q)}
    \max(-m_{pq},0).
    \end{equation}

    \item \textbf{Mono worst}

    The worst monotonicity violation is defined by
    \begin{equation}
    \label{eq:mono-worst}
    \mathrm{Mono\ worst}
    :=
    \max_{(p,q)}
    \max(-m_{pq},0).
    \end{equation}

    \item \textbf{Mono frac}

    The fraction of violating pairs is defined by
    \begin{equation}
    \label{eq:mono-frac}
    \mathrm{Mono\ frac}
    :=
    \frac{1}{N^{\rm mono}}
    \sum_{(p,q)}
    \mathbf{1}_{\{m_{pq}<0\}},
    \end{equation}
    where $\mathbf{1}$ denotes the indicator function.
\end{itemize}

Tables~\ref{tab:hyperparams-A1}
and~\ref{tab:hyperparams-plap}
summarize the choice of hyperparameters.

The results are summarized in Tables~\ref{tab:results-derivative} and ~\ref{tab:results-plap} where values are reported as mean $\pm$ standard deviation over four random seeds.

\begin{table}[t]
\centering
\caption{Hyperparameters for the experiment on \(A_1(u)=u'\).}
\label{tab:hyperparams-A1}
\begin{tabular}{ll}
\toprule
Parameter & Value \\
\midrule
Number of grid points for $[0,1]$ & \(128\) \\
Training samples \(N^{\rm train}\) & \(2000\) \\
Test samples \(N^{\rm test}\) & \(400\) \\
Batch size & \(64\) \\
Training frequency range & \(n_j\in\{2,\dots,30\}\) \\
Test frequency range & \(n_j\in\{2,\dots,30\}\) \\
Number of modes per sample \(K\) & \(K\sim \mathrm{Uniform}\{3,\dots,6\}\) \\
Frequency decay exponent \(\beta\) & \(0.5\) \\
\midrule
Epochs & \(80\) \\
Optimizer & Adam \\
Learning rate & \(2\times10^{-3}\) \\
Weight decay & \(10^{-6}\) \\
Gradient clipping threshold & \(1.0\) \\
\midrule
Soft \(\ell^\infty\) temperature \(\tau_\infty\) & \(0.02\) \\
Graph inner temperature \(\tau_{\rm gph,in}\) & \(0.01\) \\
Graph outer temperature \(\tau_{\rm gph,out}\) & \(0.01\) \\
Graph input weight \(w_1\) & \(0.5\) \\
Graph output weight \(w_2\) & \(1.0\) \\
\midrule
FNO modes & \(34\) \\
FNO width & \(100\) \\
FNO depth & \(3\) \\
\midrule
Structured model parameter \(\lambda\) & \(0.01\) \\
Non-expansive penalty weight \(\gamma\) & \(0.01\) \\
Number of non-expansive pairs \(N^{\rm nonexp}\) & \(512\) \\
Monotonicity evaluation pairs \(N^{\rm mono}\) & \(128^2\) \\
\bottomrule
\end{tabular}
\end{table}

\clearpage

\begin{table}[h]
\centering
\caption{
Hyperparameters for the experiment on 
\(
A_2(u)
=
-\operatorname{div}(|\nabla u|^{p-2}\nabla u)
\).
}
\label{tab:hyperparams-plap}
\begin{tabular}{ll}
\toprule
Parameter & Value \\
\midrule
Number of grid points for $[0,1]^2$ & \(32\times 32\) \\
Training samples \(N^{\rm train}\) & \(2000\) \\
Test samples \(N^{\rm test}\) & \(400\) \\
Batch size & \(32\) \\
Training frequency range &
\(
(k,l)\in\{2,\dots,6\}^2
\)
\\
Test frequency range &
\(
(k,l)\in\{9,\dots,16\}^2
\)
\\
Number of modes per sample \(K\) & \(K\sim \mathrm{Uniform}\{1,\dots,8\}\) \\
Frequency decay exponent \(\beta\) & \(0.0\) \\
\(p\)-Laplacian exponent \(p\) & \(1.2\) \\
\midrule
Epochs & \(100\) \\
Optimizer & Adam \\
Learning rate & \(3\times10^{-4}\) \\
Weight decay & \(10^{-6}\) \\
Gradient clipping threshold & \(0.3\) \\
\midrule
Soft \(\ell^\infty\) temperature \(\tau_\infty\) & \(0.1\) \\
Graph inner temperature \(\tau_{\rm gph,in}\) & \(10^{-4}\) \\
Graph outer temperature \(\tau_{\rm gph,out}\) & \(10^{-4}\) \\
Graph input weight \(w_1\) & \(10^{-4}\) \\
Graph output weight \(w_2\) & \(1.0\) \\
\midrule
FNO modes & \(16\) \\
FNO width & \(96\) \\
FNO depth & \(3\) \\
\midrule
Structured model parameter \(\lambda\) & \(7\times10^{-3}\) \\
Non-expansive penalty weight \(\gamma\) & \(2\times10^{-5}\) \\
Number of non-expansive pairs \(N^{\rm nonexp}\) & \(64\) \\
Monotonicity evaluation pairs \(N^{\rm mono}\) & \(5000\) \\
\bottomrule
\end{tabular}
\end{table}

\newpage

\begin{table}[h]
\centering
\footnotesize
\setlength{\tabcolsep}{3pt}

\caption{Comparison of test evaluations for \(A_1(u)=u'\).}
\label{tab:results-derivative}

\begin{tabular}{lccc}
\toprule
Model
& Test MSE.
& Mean Rel-$\ell^2$.
& Worst Rel-$\ell^2$.
\\
\midrule

$\ell^2$
& $1.1230 \pm 1.1532$
& $0.0159 \pm 0.0109$
& $0.0389 \pm 0.0314$
\\

$\ell^\infty$
& $2.3120 \pm 1.9431$
& $0.0263 \pm 0.0114$
& $0.0956 \pm 0.0532$
\\

graph
& $\mathbf{0.0421 \pm 0.0052}$
& $\mathbf{3.368\times10^{-3} \pm 1.9\times10^{-4}}$
& $\mathbf{0.0116 \pm 0.0029}$
\\

graph + structure
& $0.1784 \pm 0.0384$
& $7.934\times10^{-3} \pm 1.1\times10^{-3}$
& $0.0170 \pm 0.0047$
\\

\bottomrule
\end{tabular}

\vspace{2mm}

\begin{tabular}{lcccc}
\toprule
Model
& Test Graph.
& Mono mean-viol.
& Mono worst.
& Mono frac.
\\
\midrule

$\ell^2$
& $7.7351 \pm 10.7638$
& $20.5333 \pm 27.6013$
& $-246.1772 \pm 278.5229$
& $0.6375 \pm 0.2290$
\\

$\ell^\infty$
& $14.4008 \pm 17.1336$
& $43.5230 \pm 43.8994$
& $-493.7024 \pm 393.3084$
& $0.7730 \pm 0.1605$
\\

graph
& $\mathbf{0.2790 \pm 0.1225}$
& $1.8287 \pm 0.3608$
& $-33.9634 \pm 8.9236$
& $0.5284 \pm 0.0903$
\\

graph + structure
& $0.8975 \pm 0.3452$
& $\mathbf{0.2506 \pm 0.4498}$
& $\mathbf{-36.7055 \pm 32.1097}$
& $\mathbf{0.0111 \pm 0.0146}$
\\

\bottomrule
\end{tabular}
\end{table}

\begin{table}[h]
\centering
\footnotesize
\setlength{\tabcolsep}{3pt}

\caption{Comparison of test evaluations for \(
A_2(u)=
-\operatorname{div}(|\nabla u|^{p-2}\nabla u)
\).
}
\label{tab:results-plap}

\begin{tabular}{lccc}
\toprule
Model
& Test MSE.
& Mean Rel-$\ell^2$.
& Worst Rel-$\ell^2$.
\\
\midrule

$\ell^2$
& $1.93\times10^{4} \pm 3.0\times10^{3}$
& $6.13\times10^{-1} \pm 4.7\times10^{-2}$
& $6.88\times10^{-1} \pm 3.8\times10^{-2}$
\\

$\ell^\infty$
& $2.08\times10^{4} \pm 3.1\times10^{3}$
& $6.37\times10^{-1} \pm 4.8\times10^{-2}$
& $7.23\times10^{-1} \pm 6.5\times10^{-2}$
\\

graph
& $\mathbf{1.51\times10^{4} \pm 2.0\times10^{3}}$
& $\mathbf{5.43\times10^{-1} \pm 3.5\times10^{-2}}$
& $\mathbf{6.14\times10^{-1} \pm 4.1\times10^{-2}}$
\\

graph + structure
& $3.83\times10^{4} \pm 8.7\times10^{2}$
& $8.71\times10^{-1} \pm 8.9\times10^{-3}$
& $9.15\times10^{-1} \pm 1.4\times10^{-2}$
\\

\bottomrule
\end{tabular}

\vspace{2mm}

\begin{tabular}{lcccc}
\toprule
Model
& Test Graph.
& Mono mean-viol.
& Mono worst.
& Mono frac.
\\
\midrule

$\ell^2$
& $1.59\times10^{4} \pm 1.1\times10^{3}$
& $0.000 \pm 0.000$
& $0.000 \pm 0.000$
& $0.000 \pm 0.000$
\\

$\ell^\infty$
& $1.69\times10^{4} \pm 1.9\times10^{3}$
& $0.000 \pm 0.000$
& $0.000 \pm 0.000$
& $0.000 \pm 0.000$
\\

graph
& $\mathbf{1.28\times10^{4} \pm 1.6\times10^{3}}$
& $0.000 \pm 0.000$
& $0.000 \pm 0.000$
& $0.000 \pm 0.000$
\\

graph + structure
& $2.46\times10^{4} \pm 7.7\times10^{2}$
& $0.000 \pm 0.000$
& $0.000 \pm 0.000$
& $0.000 \pm 0.000$
\\

\bottomrule
\end{tabular}
\end{table}




\end{document}